%% file: luo23.tex
\documentclass[final, 12pt]{alt2023} 
\newenvironment{proofsk}{%
  \renewcommand{\proofname}{Proof sketch.}\proof}{\endproof}

\title[Improved High-Probability Regret for Adversarial Bandits with Feedback Graphs]{Improved High-Probability Regret for Adversarial Bandits with Time-Varying Feedback Graphs}
\usepackage{times}
\input{command}

\altauthor{%
 \Name{Haipeng Luo\nametag{\thanks{Authors are listed in alphabetical order.}}} \Email{haipengl@usc.edu}\\
 \addr University of Southern California
 \AND
 \Name{Hanghang Tong\nametag{\footnotemark[1]}} \Email{htong@illinois.edu}\\
 \addr University of Illinois at Urbana-Champaign
 \AND
 \Name{Mengxiao Zhang\nametag{\footnotemark[1]}} \Email{mengxiao.zhang@usc.edu}\\
 \addr University of Southern California
 \AND
 \Name{Yuheng Zhang\nametag{\footnotemark[1]}} \Email{yuhengz2@illinois.edu}\\
 \addr University of Illinois at Urbana-Champaign
}

\usepackage{multirow}

\begin{document}

\maketitle

\begin{abstract}%
  We study high-probability regret bounds for adversarial $K$-armed bandits with time-varying feedback graphs over $T$ rounds. For general strongly observable graphs, we develop an algorithm that achieves the optimal regret $\otil((\sum_{t=1}^T\alpha_t)^{\nicefrac{1}{2}}+\max_{t\in[T]}\alpha_t)$ with high probability, where $\alpha_t$ is the independence number of the feedback graph at round $t$. 
  Compared to the best existing result~\citep{neu2015explore} which only considers graphs with self-loops for all nodes, our result not only holds more generally, but importantly also removes any $\text{poly}(K)$ dependence that can be prohibitively large for applications such as contextual bandits.
  Furthermore, we also develop the first algorithm that achieves the optimal high-probability regret bound for weakly observable graphs, which even improves the best expected regret bound of~\citep{alon2015online2} by removing the $\order(\sqrt{KT})$ term with a refined analysis.  
  Our algorithms are based on the online mirror descent framework, but importantly with an innovative combination of several techniques.
  Notably, while earlier works use optimistic biased loss estimators for achieving high-probability bounds, we find it important to use a pessimistic one for nodes without self-loop in a strongly observable graph.
\end{abstract}

\begin{keywords}%
  multi-armed bandits, bandits with feedback graph, high-probability regret bounds
\end{keywords}

\input{intro}
\input{preliminary}
\input{strongly}
\input{weakly}
\input{conclusion}


\bibliography{alt2023-sample}

\appendix
\input{app_strongly}
\input{app_weakly}

\input{app_aux}
\end{document}

%% file: command.tex
\newif\ifspacehack
\usepackage{natbib}
\hypersetup{
    colorlinks = blue,
    breaklinks,
    linkcolor = blue,
    citecolor = blue,
    urlcolor  = blue,
}
\usepackage{graphicx}
\usepackage{mathtools}
\usepackage{footnote}
\usepackage{float}
\usepackage{xspace}
\usepackage{multirow}
\usepackage{xcolor}
\usepackage{wrapfig}
\usepackage{framed}
\usepackage{bbm}
\usepackage{footnote}
\usepackage{nicefrac}
\usepackage{makecell}
\usepackage{algorithm}
\usepackage{enumitem}
\usepackage{bm}
\makesavenoteenv{tabular}
\makesavenoteenv{table}

\renewcommand{\tilde}{\widetilde}

\newcommand{\expthree}{\textsc{Exp3}\xspace}

\newcommand{\expthreeG}{\textsc{Exp3.G}\xspace}
\newcommand{\expthreeix}{\textsc{Exp3-IX}\xspace}
\newcommand{\elpp}{\textsc{ELP.P}\xspace}

\def \R {\mathbb{R}}

\newcommand{\eps}{\epsilon}

\newcommand{\calT}{{\mathcal{T}}}

\newcommand{\Reg}{\text{\rm Reg}}

\newcommand{\one}{\boldsymbol{1}}

\newcommand{\biasterm}{\textsc{Bias-Term}\xspace}
\newcommand{\stabterm}{\textsc{Stability-Term}\xspace}

\DeclareMathOperator*{\argmin}{argmin}

\newcommand{\E}{\field{E}}

\newcommand{\inner}[1]{ \left\langle {#1} \right\rangle }

\newcommand{\wh}{\widehat}
\newcommand{\wt}{\widetilde}

\newcommand{\ellhat}{\wh{\ell}}

\newcommand{\order}{\ensuremath{\mathcal{O}}}
\newcommand{\otil}{\ensuremath{\tilde{\mathcal{O}}}}

\usepackage{lipsum,booktabs}
\usepackage{amsmath,mathrsfs,amssymb,amsfonts,bm}
\usepackage{rotating}
\usepackage{pdflscape}
\usepackage{hyperref,url}
\hypersetup{
    colorlinks,
    breaklinks,
    linkcolor = blue,
    citecolor = blue,
    urlcolor  = blue,
}
\allowdisplaybreaks
\usepackage{appendix}
\usepackage{multirow,makecell}
\usepackage[vlined]{algorithm2e}

\renewcommand{\tilde}{\widetilde}

\newcommand \term[1]{\mathtt{Term}~(\mathtt{#1})}

\def \E {\mathbb{E}}

\def \R {\mathbb{R}}

\usepackage{mathtools}

\DeclareMathOperator*{\poly}{poly}

\usepackage{prettyref}
\newcommand{\pref}[1]{\prettyref{#1}}

\newcommand{\savehyperref}[2]{\texorpdfstring{\hyperref[#1]{#2}}{#2}}
\newrefformat{eq}{\savehyperref{#1}{Eq.~\textup{(\ref*{#1})}}}
\newrefformat{eqn}{\savehyperref{#1}{Eq.~(\ref*{#1})}}
\newrefformat{lem}{\savehyperref{#1}{Lemma~\ref*{#1}}}
\newrefformat{lemma}{\savehyperref{#1}{Lemma~\ref*{#1}}}
\newrefformat{def}{\savehyperref{#1}{Definition~\ref*{#1}}}
\newrefformat{line}{\savehyperref{#1}{Line~\ref*{#1}}}
\newrefformat{thm}{\savehyperref{#1}{Theorem~\ref*{#1}}}
\newrefformat{corr}{\savehyperref{#1}{Corollary~\ref*{#1}}}
\newrefformat{cor}{\savehyperref{#1}{Corollary~\ref*{#1}}}
\newrefformat{sec}{\savehyperref{#1}{Section~\ref*{#1}}}
\newrefformat{app}{\savehyperref{#1}{Appendix~\ref*{#1}}}
\newrefformat{appendix}{\savehyperref{#1}{Appendix~\ref*{#1}}}
\newrefformat{assum}{\savehyperref{#1}{Assumption~\ref*{#1}}}
\newrefformat{ex}{\savehyperref{#1}{Example~\ref*{#1}}}
\newrefformat{fig}{\savehyperref{#1}{Figure~\ref*{#1}}}
\newrefformat{table}{\savehyperref{#1}{Table~\ref*{#1}}}
\newrefformat{tab}{\savehyperref{#1}{Table~\ref*{#1}}}
\newrefformat{alg}{\savehyperref{#1}{Algorithm~\ref*{#1}}}
\newrefformat{rem}{\savehyperref{#1}{Remark~\ref*{#1}}}
\newrefformat{remark}{\savehyperref{#1}{Remark~\ref*{#1}}}
\newrefformat{conj}{\savehyperref{#1}{Conjecture~\ref*{#1}}}
\newrefformat{prop}{\savehyperref{#1}{Proposition~\ref*{#1}}}
\newrefformat{proto}{\savehyperref{#1}{Protocol~\ref*{#1}}}
\newrefformat{prob}{\savehyperref{#1}{Problem~\ref*{#1}}}
\newrefformat{claim}{\savehyperref{#1}{Claim~\ref*{#1}}}
\newrefformat{que}{\savehyperref{#1}{Question~\ref*{#1}}}
\newrefformat{op}{\savehyperref{#1}{Open Problem~\ref*{#1}}}
\newrefformat{fn}{\savehyperref{#1}{Footnote~\ref*{#1}}}

\def \epsilon {\varepsilon}

\newcommand{\Tcal}{\mathcal T}

\newcommand{\Nin}{N^{\mathrm{in}}}

\newcommand{\pa}[1]{\left(#1\right)}

\newcommand{\brgmd}{D}
\newcommand{\hatell}{\widehat{\ell}}

\renewcommand{\ln}{\log}

%% file: intro.tex

\section{Introduction}\label{sec: intro}
In this work, we study adversarial multi-armed bandits (MAB) with directed feedback graphs, which is a generalization of the expert problem~\citep{freund1997decision} and the standard MAB problem~\citep{auer2002nonstochastic}. The interaction between the learner and the environment lasts for $T$ rounds. In each round, the learner needs to choose one of $K$ actions while simultaneously an adversary decides the loss for each action. After that, the learner suffers the loss of the chosen action, and her observation is determined based on a directed graph with the $K$ actions as nodes. Specifically, she observes the loss of every action to which the chosen action is connected. When the graph only contains self-loops, this recovers the standard MAB problem, and when the graph is a complete graph, this recovers the expert problem. By allowing arbitrary feedback graphs, however, this model captures many other interesting problems; see~\citep{mannor2011bandits} for example.

\citet{alon2015online2} characterized the minimax expected regret bound for this problem with a fixed feedback graph $G$.
Specifically, for a strongly observable graph (see~\pref{sec: problem-setup} for all formal definitions), their algorithm achieves $\otil(\sqrt{\alpha T})$ expected regret where $\alpha$ is the independence number of $G$, while for a weakly observable graph, they achieve $\otil(d^{\nicefrac{1}{3}}T^{\nicefrac{2}{3}})$ expected regret (ignoring a $\order(\sqrt{KT})$ term), where $d$ is the weak domination number of $G$. Both are shown to be near-optimal.

Despite these near-optimal expected regret guarantees, it is known that these algorithm exhibit a huge variance and can in fact suffer $\Theta(T)$ regret with a constant probability (see~\citep{lattimore2020bandit}), which is clearly undesirable in practice.
To mitigate this issue, \citet{alon2017nonstochastic} designed an algorithm called~\elpp, which ensures $\otil(\sqrt{mT}+m^2T^{\nicefrac{1}{4}})$ regret with high probability for self-aware graphs (a special case of strongly observable graphs in which every node has a self-loop), where $m$ is the size of the maximal acyclic graph in $G$ and can be much larger than $\alpha$. 
On the other hand, \citet{neu2015explore} developed the \expthreeix algorithm which uses implicit exploration in the loss estimator construction and achieves $\otil(\sqrt{\alpha T}+K)$ high-probability regret bound also for self-aware graphs. 
While the bound is almost optimal, the additional $K$ term could be prohibitively large for applications such as contextual bandits where $K$ is the number of policies (usually considered as exponentially large).

\begin{table}[t]
\centering
\caption{\small Summary of our results and comparisons with prior work. $T$ is the number of rounds. $K$ is the number of actions. $\alpha_t$ and $d_t$ are respectively the independence number and the weak domination number of feedback graph $G_t$ at round $t$. The results of~\citep{alon2015online2,neu2015explore} are for a fixed feedback graph $G$ (so $G_t = G$, $\alpha_t = \alpha$, and $d_t = d$ for all $t$). 
Our high-probability regret bound for weakly observable graphs omits some lower-order terms; see~\pref{thm: weak_high_prob} for the complete form.}
\label{tab: table1}
\renewcommand*{\arraystretch}{1.4}
\resizebox{\textwidth}{!}{
\begin{tabular}{|cc|c|cc|}
\hline
\multicolumn{2}{|c|}{\multirow{2}{*}{Graph Type}}                       & Expected Regret            & \multicolumn{2}{c|}{High-probability Regret}                \\ \cline{3-5} 
\multicolumn{2}{|c|}{}                                                  & \citep{alon2015online2}                     & \multicolumn{1}{c|}{\citep{neu2015explore}}     & \textbf{Our work }                     \\ \hline
\multicolumn{1}{|c|}{\multirow{2}{*}{Strongly Observable}} & General    & \multirow{2}{*}{$\otil(\sqrt{\alpha T})$} & \multicolumn{1}{c|}{N/A}   & \multirow{2}{*}{$\otil\left(\sqrt{\sum_{t=1}^T\alpha_t}+\max_{t\in[T]}\alpha_t\right)$} \\ \cline{2-2} \cline{4-4}
\multicolumn{1}{|c|}{}                                     & Self-aware &                         & \multicolumn{1}{c|}{$\otil(\sqrt{\alpha T}+K)$} &                         \\ \hline
\multicolumn{2}{|c|}{Weakly Observable}                                 & $\otil(d^{\nicefrac{1}{3}}T^{\nicefrac{2}{3}}+\sqrt{KT})$                 & \multicolumn{1}{c|}{N/A}   & $\otil((\sum_{t=1}^Td_t)^{\nicefrac{1}{3}}T^{\nicefrac{1}{3}}+\frac{1}{T}\sum_{t=1}^T d_t)$                  \\ \hline
\end{tabular}}
\end{table}


In this work, we significantly improve these results and extend them to more general graphs.
For full generality, we also consider a sequence of time-varying feedback graphs $G_1, \ldots, G_T$, each of which can be chosen adaptively by the environment based on the learner's previous actions. We denote the independence number of $G_t$ by $\alpha_t$ and its weakly domination number by $d_t$. Our main contributions are (see also \pref{tab: table1}):
\begin{itemize}[leftmargin=10pt]
    \item In \pref{sec: strongly}, we start with a refined analysis showing that~\expthreeix of~\citep{neu2015explore} in fact achieves $\otil((\sum_{t=1}^T\alpha_t)^{\nicefrac{1}{2}}+\max_{t\in[T]}\alpha_t)$ high-probability regret bound for self-aware graphs, removing the $\otil(K)$ dependence of~\citep{neu2015explore}. We then extend the same bound to the more general strongly observable graphs via a new algorithm that, on top of the implicit exploration technique of \expthreeix, further injects certain positive bias to the loss estimator of an action that has no self-loop but is selected with more than $1/2$ probability, making it a pessimistic estimator.
    
    \item In \pref{sec: weakly}, we propose an algorithm with high-probability regret $\otil((\sum_{t=1}^Td_t)^{\nicefrac{1}{3}}T^{\nicefrac{1}{3}}+\frac{1}{T}\sum_{t=1}^T d_t)$ for weakly observable graphs (ignoring some lower-order terms). To the best of our knowledge, this is the first algorithm with (near-optimal) high-probability regret guarantees for such graphs. Moreover, our bound even improves the expected regret bound of~\citep{alon2015online2} by removing the $\otil(\sqrt{KT})$ term. 
\end{itemize}

We remark that for simplicity we prove our results by assuming the knowledge of $\alpha_1, \ldots, \alpha_T$ or $d_1, \ldots, d_T$ to tune the parameters, but this can be easily relaxed using the standard doubling trick, making our algorithms completely parameter-free.

\paragraph{Techniques.} Our algorithms are based on the well-known Online Mirror Descent (OMD) framework with the entropy regularizer. However, several crucial techniques are needed to achieve our results, including implicit exploration, explicit uniform exploration, injected positive bias, and a loss shifting trick. Among them, using positive bias and thus a pessimistic loss estimator is especially notable since most earlier works use optimistic underestimators for achieving high-probability regret bounds.
The combination of these techniques also requires non-trivial analysis. 



\paragraph{Related Works.} Since~\citet{mannor2011bandits} initiated the study of online learning with feedback graphs, many follow-up works consider different variants of the problem, including stochastic feedback graphs~\citep{caron2012leveraging,buccapatnam2017reward,marinov2022stochastic}, minimax regret bounds for different feedback graph types~\citep{alon2015online2,chen2021understanding}, small-loss bounds~\citep{lykouris2018small,lee2020closer}, best-of-both-world algorithms~\citep{erez2021towards,ito2022nearly}, and uninformed time-varying feedback graphs~\citep{cohen2016online}. 


This work focuses on achieving high-probability regret bounds, which is relatively less studied in the bandit literature but as mentioned extremely important due to the potentially large variance of the regret.
As far as we know, to achieve high-probability regret bounds for adversarial bandit problems, there are three categories of methods as discussed below.

The first method is to inject a negative bias to the loss estimators, making them optimistic and trading unbiasedness for lower variance. 
Examples include the very first work in this line for standard MAB~\citep{auer2002nonstochastic}, 
linear bandits~\citep{bartlett2008high, abernethy2009beating, zimmert2022return}, and bandits with self-aware feedback graphs~\citep{alon2017nonstochastic}.

The second method is the so-called implicit exploration approach~\citep{kocak2014efficient} (which also leads to optimistic estimators). \citet{neu2015explore} used this method to achieve $\otil(\sqrt{KT})$ regret for MAB and $\otil(\sqrt{\alpha T}+K)$ regret for bandits with a fixed self-aware feedback graph, improving over the results by~\citep{alon2017nonstochastic}. 
\citet{lykouris2018small} also used implicit exploration and achieved high-probability first-order regret bound for bandits with self-aware undirected feedback graphs. However, their regret bounds are either suboptimal in $T$ or in terms of the clique partition number of the graph (which can be much larger than the independence number).

The third method is to use OMD with a self-concordant barrier and an increasing learning rate scheduling, proposed by~\citet{lee2020bias}.
They used this method to achieve high-probability data-dependent regret bounds for MAB, linear bandits, and episodic Markov decision processes. However, using a self-concordant barrier regularizer generally leads to $\otil(\sqrt{KT})$ regret in bandits with strongly observable feedback graphs and $\otil(K^{\nicefrac{1}{3}}T^{\nicefrac{2}{3}})$ regret in bandits with weakly observable feedback graphs, making it suboptimal compared to the minimax regret bound $\otil(\sqrt{\alpha T})$ and $\otil(d^{\nicefrac{1}{3}}T^{\nicefrac{2}{3}})$ respectively. 

All our algorithms adopt the implicit exploration technique for nodes with self-loop.
For strongly observable graphs, we find it necessary to further adopt the injected bias idea for nodes without self-loop, but contrary to prior works, our bias is positive, which makes the loss overestimated and intuitively prevents the algorithm from picking such nodes too often without seeing their actual loss.



%% file: preliminary.tex

\section{Problem Setup and Notations}\label{sec: problem-setup}
Throughout the paper, we denote $\{1,2,\cdots,N\}$ by $[N]$ for some positive integer $N$. At each round $t\in [T]$, the learner selects one of the $K$ available actions $i_t\in [K]$, while the adversary decides a loss vector $\ell_t\in[0,1]^K$ with $\ell_{t,i}$ being the loss for action $i$, and a directed feedback graph $G_t=([K], E_t)$ where $E_t\subseteq [K]\times[K]$. The adversary can be adaptive and chooses $\ell_t$ and $G_t$ based on the learner's previous actions $i_1,\dots,i_{t-1}$ in an arbitrary way. At the end of round $t$, the learner observes some information about $\ell_t$ according to the feedback graph $G_t$. Specifically, she observes the loss of action $j$ for all $j$ such that $i_t\in \Nin_t(j)$, where $\Nin_t(j)=\{i \in [K]: (i, j)\in E_t\}$ is the set of nodes that can observe node $j$. The standard measure of the
learner's performance is the regret, defined as the
difference between the total loss of the learner and that of the best fixed action in hindsight
\begin{align*}
    \Reg \triangleq \sum_{t=1}^T\ell_{t,i_t}-\sum_{t=1}^T\ell_{t,i^*},
\end{align*}
where $i^*=\argmin_{i\in[K]}\sum_{t=1}^T\ell_{t,i}$. In this work, we focus on designing algorithms with \emph{high-probability} regret guarantees. 

We refer the reader to~\citep{alon2015online2} for the many examples of such a general model, and only point out that the contextual bandit problem~\citep{langford2007epoch} is indeed a special case where each node corresponds to a policy and each $G_t$ is the union of several cliques. Each such clique consists of all polices that make the same decision for the current context at round $t$.
In this case, $K$, the number of policies, is usually considered as exponentially large, and only $\text{polylog}(K)$ dependence on the regret is acceptable.
This justifies the significance of our results that indeed remove $\text{poly}(K)$ dependence from existing regret bounds.

\paragraph{Strongly/Weakly Observable Graphs.} For a directed graph $G=([K], E)$, a node $i$ is observable if $\Nin(i)\neq\emptyset$. An observable node is strongly observable
if either $i\in \Nin(i)$ or $\Nin(i) = [K]\backslash\{i\}$, and weakly observable otherwise. Similarly, a graph is observable if all its nodes are observable. An observable graph is \emph{strongly observable} if all nodes
are strongly observable, and \emph{weakly observable} otherwise. \emph{Self-aware} graphs are a special type of strongly observable graphs where $i\in \Nin(i)$ for all $i\in[K]$.

\paragraph{Independent Set and Weakly Dominating Set.} An independence set of a directed graph is a subset of nodes in which no two distinct nodes are connected. The size of the largest independence set in $G_t$, called the \textit{independence number} of $G_t$, is denoted by $\alpha_t$. For a weakly observable graph $G$, a weakly dominating set is a subset $\mathcal{D}$ of nodes such that for any node $j$ in $G$ without self-loop, there exists $i \in \mathcal{D}$ such that $i$ is connected to $j$. 
The size of the smallest weakly dominating set of $G_t$, called the \textit{weak domination number} of $G_t$, is denoted by $d_t$.\footnote{We follow the definition in~\citep{ito2022nearly}, which differs by at most $1$ compared to that in~\citep{alon2015online2}.} 

\paragraph{Informed/Uninformed Setting.} 
Under the \emph{informed setting}, the feedback graph $G_t$ is shown to the learner at the beginning of round $t$ before she selects $i_t$. In other words, the learner's decision at round $t$ can be dependent on $G_t$. In contrast, the \emph{uninformed setting} is a harder setting, in which the learner observes $G_t$ only at the end of round $t$ after she selects $i_t$.
For strongly observable graphs, we study the harder uninformed setting,
while for weakly observable graphs, in light of the $\Omega(K^{\nicefrac{1}{3}}T^{\nicefrac{2}{3}})$ regret lower bound of~\citep[Theorem 9]{alon2015online}, we only study the informed setting.

\paragraph{Other Notations.} 
Define $S_t \triangleq \{i \in [K]: i \in \Nin_t(i)\}$ as the set of nodes with self-loop in $G_t$. 
For a differentiable convex function $\psi$ defined on a convex set $\Omega$, we denote the induced Bregman divergence by $D_\psi(x, y) = \psi(x)-\psi(y)-\inner{\nabla\psi(y), x-y}$ for any two points
$x,y\in\Omega$. For notational convenience, for two vectors $x,y\in\mathbb{R}^K$ and an arbitrary index set $U\subseteq [K]$, we define $\inner{x,y}_U\triangleq\sum_{i\in U}x_iy_i$ to be the partial inner product with respect to the coordinates in $U$. We denote the $(K-1)$-dimensional simplex by $\Delta_K$, the all-one vector in $\mathbb{R}^K$ by $\one$, and the $i$-th standard basis vector in $\mathbb{R}^K$ by $e_i$. We use the $\otil(\cdot)$ notation to hide factors that are logarithmic in $K$ and $T$.\footnote{In the text, $\order(\cdot)$ and $\otil(\cdot)$ often further hide lower-order terms (in terms of $T$ dependence) and $\poly(\log(1/\delta))$ factors for simplicity. However, in all formal theorem/lemma
statements, we use $\order(\cdot)$ to hide universal constants only and $\otil(\cdot)$ to also hide factors logarithmic in $K$ and $T$.}

%% file: strongly.tex

\section{Optimal High-Probability Regret for Strongly Observable Graphs}\label{sec: strongly}
In this section, we consider the uninformed setting with strongly observable graphs, that is, each $G_t$ is strongly observable and revealed to the learner after she selects $i_t$ at round $t$. We propose an algorithm which achieves $\otil((\sum_{t=1}^T\alpha_t)^{\nicefrac{1}{2}}+\max_{t\in[T]}\alpha_t)$ high-probability regret bound. 
As mentioned, this result improves over those from~\citep{neu2015explore, alon2017nonstochastic} in two aspects: first, they only consider self-aware graphs;\footnote{Although~\citep{neu2015explore} only considers a fixed feedback graph (i.e. $G_t=G$ for all $t\in[T]$), their result can be directly generalized to time-varying feedback graphs. On the other hand, we point out that~\citep{alon2017nonstochastic} only considers the easier informed setting.}
second, our bound enjoys the optimal $(\sum_{t=1}^T\alpha_t)^{\nicefrac{1}{2}}$ dependence with no $\poly(K)$ dependence at all.

To present our algorithm, which is built on top of the \expthreeix algorithm of~\citep{neu2015explore},
we start by reviewing how \expthreeix works and how it achieves $\otil((\sum_{t=1}^T\alpha_t)^{\nicefrac{1}{2}}+K)$ high-probability regret bound for self-aware graphs.
At each round $t$, after picking the action $i_t$ randomly according to $p_t\in \Delta_K$ and observing the loss $\ell_{t,j}$ for all $j$ such that $i_t\in \Nin_t(j)$, \expthreeix constructs the underestimator $\ellhat_{t}$ for $\ell_t$, such that $\ellhat_{t,i}=\frac{\ell_{t,i}}{W_{t,i}+\gamma}\cdot\one\{i_t\in\Nin_t(i)\}$ where $W_{t,i}=\sum_{j\in\Nin_t(i)}p_{t,j}$ is the probability of observing $\ell_{t,i}$ and $\gamma>0$ is a bias parameter. Then, the strategy at round $t+1$ is computed via the standard multiplicative weight update (equivalent to OMD with entropy regularizer): $p_{t+1,i}\propto p_{t,i}\exp(-\eta\ellhat_{t,i})$ for all $i\in[K]$ where $\eta>0$ is the learning rate. 

Following standard analysis of OMD, we know that for any $j\in [K]$,
\begin{align}\label{eqn: omd-1}
    \sum_{t=1}^T\inner{p_{t}-e_j, \ellhat_{t}} &\leq \underbrace{\frac{\log K}{\eta}}_{\biasterm} + \underbrace{\eta\sum_{t=1}^T\sum_{i=1}^Kp_{t,i}\ellhat_{t,i}^2}_{\stabterm}.
\end{align}
To derive the high-probability regret bound from~\pref{eqn: omd-1}, \citet{neu2015explore} first shows that with probability at least $1-\delta$, the following two inequalities hold due to the underestimation:
\begin{align}
    &\sum_{t=1}^T\left(\ellhat_{t,j}-\ell_{t,j}\right)\leq \frac{\log(2K/\delta)}{2\gamma}, \quad\forall j\in[K] \label{eqn: ix-1} \\
    &\sum_{t=1}^T\sum_{i=1}^K\frac{p_{t,i}}{W_{t,i}+\gamma}\left(\ellhat_{t,i}-\ell_{t,i}\right)\leq \frac{K\log(2K/\delta)}{2\gamma}. \label{eqn: ix-2}
\end{align}
Define $Q_t\triangleq \sum_{i\in S_t}\frac{p_{t,i}}{W_{t,i}+\gamma}$, which is simply $\sum_{i=1}^K\frac{p_{t,i}}{W_{t,i}+\gamma}$ for self-aware graphs. Using~\pref{eqn: ix-2}, the stability term can be upper bounded as follows:
\begin{align}
    \stabterm \leq \eta\sum_{t=1}^T\sum_{i=1}^K\frac{p_{t,i}}{W_{t,i}+\gamma}\ellhat_{t,i}\leq \eta\sum_{t=1}^TQ_t + \otil\left(\frac{K\eta}{\gamma}\right). \label{eqn:stability}
\end{align}
To connect the true regret $\sum_{t=1}^T(\ell_{t,i_t}-\ell_{t,i^*})$ with $\sum_{t=1}^T\inner{p_{t}-e_{i^*}, \ellhat_t}$, direct calculation shows:
\begin{align}\label{eqn: decompose-ix}
    \sum_{t=1}^T(\ell_{t,i_t}-\ell_{t,i^*}) &= \sum_{t=1}^T\inner{p_{t}-e_{i^*},\ellhat_t} + \sum_{t=1}^T\left(\ell_{t,i_t}-\inner{p_t,\ell_t}\right) + \sum_{t=1}^T\left(\ellhat_{t,i^*}-\ell_{t,i^*}\right) \nonumber\\
    &\qquad + \sum_{t=1}^T\sum_{i=1}^K\left(W_{t,i}-\one\{i_t\in\Nin_t(i)\}\right)\frac{p_{t,i}\ell_{t,i}}{W_{t,i}+\gamma} + \sum_{t=1}^T\sum_{i=1}^K\frac{\gamma p_{t,i}\ell_{t,i}}{W_{t,i}+\gamma}.
\end{align}
In this last expression (summation of five terms), the first term is bounded using~\pref{eqn: omd-1} and~\pref{eqn:stability}; the second term is upper bounded by $\otil(\sqrt{T})$ via standard Azuma's inequality; the third term is bounded by $\otil(1/\gamma)$ according to~\pref{eqn: ix-1}; the fourth term is a summation over a martingale sequence and can be bounded by $\otil\left(\sqrt{\sum_{t=1}^TQ_t}+K\right)$ with high probability via Freedman's inequality; and the last term can be bounded by $\gamma\sum_{t=1}^TQ_t$. Combining all the bounds above, we obtain that with high probability, the regret is bounded as follows:
\begin{align*}
    \sum_{t=1}^T\left(\ell_{t,i_t}-\ell_{t,i^*}\right)\leq \otil\left(\frac{1}{\eta}+\frac{\eta K}{\gamma}+\frac{1}{\gamma}+\sqrt{\sum_{t=1}^TQ_t}+K+(\eta+\gamma) \sum_{t=1}^TQ_t\right).
\end{align*}
Finally, using the fact that $Q_t = \otil(\alpha_t)$ (Lemma 1 of~\citep{kocak2014efficient}, included as~\pref{lem: graph} in this work) and choosing $\gamma$ and $\eta$ optimally gives $\otil((\sum_{t=1}^T\alpha_t)^{\nicefrac{1}{2}}+K)$ high-probability bound. 

\paragraph{Improvement from $\otil(K)$ to $\otil(\max_{t\in[T]}\alpha_t)$.} We now show that with a refined analysis, the undesirable $\otil(K)$ dependence can be improved to $\otil(\max_{t\in[T]}\alpha_t)$ (still for self-aware graphs using the same \expthreeix algorithm). From the previous analysis sketch of~\citep{neu2015explore}, we can see that the $\otil(K)$ dependency comes from two terms: \stabterm and the fourth term in~\pref{eqn: decompose-ix}. The upper bound of \stabterm is derived by using~\pref{eqn: ix-2} and the fourth term in~\pref{eqn: decompose-ix} is bounded via Freedman's inequality. 
We show that both of these two bounds are in fact loose and can be improved by using a strengthened Freedman's inequality (Lemma 9 of~\citep{zimmert2022return}, included as~\pref{lem: Strengthened-Freedman} in the appendix). 
Specifically, we prove the following lemma to bound these two terms. Note that this lemma is not restricted to self-aware graphs, and we will use it later for both general strongly observable graphs and weakly observable graphs.

\begin{lemma}\label{lem: stab_freedman}
    For all $t$ and $i\in S_t$, let $\ellhat_{t,i}$ be the underestimator $\frac{\ell_{t,i}}{W_{t,i}+\gamma}\cdot\mathbbm{1}\{i_t\in\Nin_t(i)\}$ with $\gamma \leq \frac{1}{2}$. Then, with probability at least $1-\delta$, the following two inequalities hold:
    \begin{align}
    &\sum_{t=1}^T\sum_{i\in S_t}\frac{p_{t,i}}{W_{t,i}+\gamma}\left(\ellhat_{t,i}-\ell_{t,i}\right) \leq \order\left(\sum_{t=1}^T\frac{Q_t^2}{\gamma U_T} + U_T\log\left(\frac{KT}{\delta\gamma}\right)\right),\label{eqn: alpha-eqn-1}\\
    &\sum_{t=1}^T\sum_{i\in S_t}\left(W_{t,i}-\one\{i_t\in \Nin_t(i)\}\right)\frac{p_{t,i}\ell_{t,i}}{W_{t,i}+\gamma} \leq \order\left(\sqrt{\sum_{t=1}^TQ_t\iota_1} + \max_{t\in[T]}Q_t\iota_1\right),\label{eqn: alpha-eqn-2}
    \end{align}
    where $Q_t=\sum_{i\in S_t}\frac{p_{t,i}}{W_{t,i}+\gamma}$, $U_t = \max\left\{1, \frac{2\max_{t\in[T]}Q_t}{\gamma}\right\}$,  and $\iota_1=\log\left(\frac{2\max_tQ_t+2\sqrt{\sum_{t=1}^TQ_t}}{\delta}\right)$.
\end{lemma}
The full proof is deferred to~\pref{app: proof-exp3-ix}. As $Q_t/\gamma\leq U_T=\Theta(\max_{t\in[T]}Q_t)/\gamma$ for all $t\in[T]$ and $Q_t\leq \otil(\alpha_t)$ , \pref{lem: stab_freedman} shows that \stabterm is bounded by $\otil(\eta\sum_{t=1}^TQ_t + \eta U_T) =\otil(\eta \sum_{t=1}^T\alpha_t +\eta\max_{t\in[T]}\alpha_t/\gamma)$, which only has logarithmic dependence on $K$, unlike the $\otil(\eta K/\gamma)$ bound of~\pref{eqn: ix-2}! For the fourth term in~\pref{eqn: decompose-ix},~\pref{lem: stab_freedman} shows that it is bounded by $\otil((\sum_{t=1}^T\alpha_t)^{\nicefrac{1}{2}}+\max_{t\in[T]}\alpha_t)$, which again has no $\text{poly}(K)$ dependence.  Combining~\pref{lem: stab_freedman} with the rest of the analysis outlined earlier, we know that \expthreeix in fact achieves $\otil((\sum_{t=1}^T\alpha_t)^{\nicefrac{1}{2}}+\max_{t\in[T]}\alpha_t)$ high-probability regret for self-aware graphs, formally stated in the following theorem. The full proof is deferred to~\pref{app: proof-exp3-ix}.
\begin{theorem}\label{thm: exp3-ix}
    \expthreeix with the optimal choice of $\eta>0$ and $\gamma>0$ guarantees that with probability at least $1-\delta$, 
$
        \Reg= \otil\left(\sqrt{\sum_{t=1}^T\alpha_t\log\frac{1}{\delta}}+\max_{t\in[T]}\alpha_t\log\frac{1}{\delta}\right).
$
\end{theorem}

\begin{algorithm}[t]
\caption{Algorithm for Strongly Observable Graphs}\label{alg: strong-obs}
\textbf{Input:} Parameter $\gamma$, $\beta$, $\eta$, $\calT=\emptyset$.

\textbf{Define:} Regularizer $\psi(p)=\frac{1}{\eta}\sum_{i=1}^K p_i \ln p_i$.

\textbf{Initialize:} $p_1$ is such that $p_{1,i} = \frac{1}{K}$ for all $i\in [K]$.

\For{$t=1,2,\dots, T$}{

\nl Calculate $\wt{p}_{t} = (1-\eta)p_{t}+\frac{\eta}{K}\one$.\label{line: uniform}

\nl Sample action $i_t$ from $\wt{p}_t$.

\nl Receive the feedback graph $G_t$ and the feedback $\ell_{t,j}$ for all $j$ such that $i_t\in \Nin_t(j)$.

\nl
Construct estimator $\ellhat_t \in\R^K$ such that 
$\ellhat_{t,i}=\frac{\ell_{t,i}\mathbbm{1}\{i_t\in \Nin_t(i)\}}{W_{t,i}+\gamma\mathbbm{1}\{i\in S_t\}}$
where $W_{t,i}=\sum_{j\in \Nin_t (i)}\wt{p}_{t,j}$. \label{line:strong-estimator}

\nl If there exists a node $j_t\in \bar{S_t}$ with $\wt{p}_{t,j_t}> \frac{1}{2}$ (at most one such $j_t$ exists), set $\calT \leftarrow \{t\}\cup\calT$.

\nl
Construct bias $b_t\in\R^K$ such that 
$b_{t,i} = \frac{\beta}{W_{t,i}}\mathbbm{1}\{t \in \calT, i = j_t\}$. \label{line:strong-bias}

\nl
Compute $p_{t+1}=\argmin_{p\in\Delta_K}\big\{\langle p,\ellhat_{t}+b_{t}\rangle+{\brgmd}_{\psi}(p,p_{t})\big\}$.
}\label{line:strong-omd}
\end{algorithm}

\paragraph{Generalization to Strongly Observable Graphs.} Next, we show how to deal with general strongly observable graphs. The pseudocode of our proposed algorithm is shown in~\pref{alg: strong-obs}. Compared to \expthreeix, there are three main differences. First, we have an additional $\eta$ amount of uniform exploration over all actions~(\pref{line: uniform}). Second, while keeping the same loss estimator construction for node $i\in S_t$ at each round $t$, for $i\in\bar{S}_t$ (nodes without self-loop), we construct a standard unbiased estimator~(\pref{line:strong-estimator}). Third, if there exists $j_t\in\bar{S}_t$ such that the probability of choosing action $j_t$ is larger than $\frac{1}{2}$, then we further add positive bias $\beta/W_{t,j_t}$ to the loss estimator $\ellhat_{t,j_t}$ (encoded via the $b_t$ vector; see~\pref{line:strong-bias} and~\pref{line:strong-omd}), making it a pessimistic over-estimator.
Intuitively, the reason of doing so is that we should avoid picking actions without self-loop too often even if past data suggest that it is a good action, since the only way to observe its loss and see if it stays good is by selecting other actions.
A carefully chosen positive bias injected to the loss estimator of such actions would exactly allow us to achieve this goal.
In what follows, by outlining the analysis of our algorithm, we further explain why we make each of these three modifications from a technical perspective.

First, we note that with a nonempty $\bar{S}_t$, the key fact used earlier $\sum_{i=1}^K\frac{p_{t,i}}{W_{t,i}+\gamma} = \otil(\alpha_t)$ is no longer true, making the fourth and the fifth term in~\pref{eqn: decompose-ix} larger than desired if we still do implicit exploration for all nodes.
Therefore, for nodes in $\bar{S}_t$, we go back to the original inverse importance weighted \emph{unbiased} estimators~(\pref{line:strong-estimator}), and decompose the regret against any fixed action $j\in[K]$ differently into the following six terms:
\begin{align}\label{eqn: strong-regret-full-decompose}
    \sum_{t=1}^T\left(\ell_{t,i_t}-\ell_{t,j}\right)  &\leq \underbrace{\left(\sum_{t=1}^T\ell_{t,i_t}-\sum_{t=1}^T\inner{\wt{p}_t, \ell_t}\right)}_{\term{1}} + \underbrace{\left(\sum_{t=1}^T\inner{\wt{p}_t-p_t,\ell_t}\right)}_{\term{2}}\nonumber\\
    &\qquad + \underbrace{\left(\sum_{t=1}^T\inner{p_t-e_{j}, \ell_t}_{\bar{S}_t} - \sum_{t=1}^T\mathbb{E}_t\left[\inner{p_t-e_{j}, \ellhat_t}_{\bar{S}_t}\right] \right)}_{\term{3}} \nonumber\\
    &\qquad + \underbrace{\left(\sum_{t=1}^T\mathbb{E}_t\left[\inner{p_t-e_{j}, \ellhat_t}_{\bar{S}_t}\right] - \sum_{t=1}^T\inner{p_t-e_{j}, \ellhat_t}_{\bar{S}_t}\right)}_{\term{4}} \nonumber\\
    &\qquad + \underbrace{\left(\sum_{t=1}^T\inner{p_t-e_{j},\ell_t-\ellhat_t}_{S_t}\right)}_{\term{5}} + \underbrace{\left(\sum_{t=1}^T\inner{p_t-e_{j}, \ellhat_t}\right)}_{\term{6}}.
\end{align}
$\term{1}$ can be bounded again by $\otil(\sqrt{T})$ via standard Azuma's inequality. $\term{2}$ can be bounded by $\eta T$ due to the $\order(\eta)$ amount of uniform exploration. $\term{3}$ is simply $0$ as $\ellhat_{t,i}$ is unbiased for $i\in\bar{S}_t$. $\term{5}$ can similarly be written as $\sum_{t=1}^T(\ellhat_{t,j}-\ell_{t,j})\mathbbm{1}\{j\in S_t\} + \sum_{t=1}^T\sum_{i\in S_t}\left(W_{t,i}-\one\{i_t\in\Nin_t(i)\}\right)\frac{p_{t,i}\ell_{t,i}}{W_{t,i}+\gamma} + \sum_{t=1}^T\sum_{i\in S_t}\frac{\gamma p_{t,i}\ell_{t,i}}{W_{t,i}+\gamma}$ since the loss estimator construction for $i\in S_t$ is the same as the one in~\expthreeix.  Following how we handled the last three terms in~\pref{eqn: decompose-ix}, we have with high probability,
\begin{align}\label{eqn: term-4-bound-main}
    \term{5}\leq \otil\left(\sqrt{\sum_{t=1}^T\alpha_t}+\max_{t\in[T]}\alpha_t + \gamma \sum_{t=1}^T\alpha_t+\frac{1}{\gamma}\right).
\end{align}
The formal statement and the proof are deferred to~\pref{lem: term-5} in~\pref{app: proof-strong-alg}.

The key challenge lies in controlling $\term{4}$ and~$\term{6}$.
To this end, let us first consider the variance of the $\langle p_t-e_j, \ellhat_t \rangle_{\bar{S}_t}$. Let $\calT=\{t: \exists j\in \bar{S}_t, \wt{p}_{t,j_t}>\frac{1}{2}\}$ be the final value of $\calT$ in \pref{alg: strong-obs}. 
If $t\notin \calT$, then $\ellhat_{t,i}\leq \frac{1}{1-\wt{p}_{t,i}}\leq 2$ for all $i\in\bar{S}_t$ and the variance of $\langle p_t-e_j, \ellhat_t \rangle_{\bar{S}_t}$ is a constant; otherwise, we know that there is at most one node $j_t\in\bar{S}_t$ such that $p_{t,j_t}>\frac{1}{2}$. Direct calculation shows that the variance of $\langle p_t-e_j, \ellhat_t \rangle_{\bar{S}_t}$ is bounded by $\otil(\frac{1}{W_{t,j_t}}\cdot\mathbbm{1}\{j\ne j_t\})$. With the help of the Freedman's inequality and the fact $W_{t,j_t}=\Omega(\eta)$ thanks to the uniform exploration (\pref{line: uniform}), we prove in \pref{lem: term-4} of \pref{app: proof-strong-alg} that $\term{4}$ can be bounded as follows: 
\begin{align}\label{eqn: term-2-bound}
    \term{4}\leq \otil\left(\sqrt{\sum_{t=1}^T\left(1+\frac{\mathbbm{1}\{j\neq j_t, t\in\calT\}}{W_{t,j_t}}\right)}+\frac{1}{\eta}\right).
\end{align}
Handling this potentially large deviation is exactly the reason we inject a positive bias to the loss estimator~(\pref{line:strong-bias}). Specifically, when $t\in \calT$, we add $b_{t,j_t}=\frac{\beta}{W_{t,j_t}}$ to the loss estimator $\ellhat_{t,j_t}$ for some parameter $\beta>0$. 
With the help of this positive bias, we can decompose $\term{6}$ as follows:
\begin{align}\label{eqn: bias-decompose}
    \term{6} = \sum_{t=1}^T\inner{p_t-e_{j}, \ellhat_t+b_t} - \sum_{t=1}^T\inner{p_t-e_{j},b_t}.
\end{align}
Direct calculation shows that the second negative term is of order $-\Theta(\sum_{t\in\calT}\frac{\beta}{W_{t,j_t}}\cdot\mathbbm{1}\{j\neq j_t\}))+\Theta(\sum_{t\in\calT}\beta\cdot\mathbbm{1}\{j=j_t\})$, large enough to cancel the large deviation in~\pref{eqn: term-2-bound}. Specifically, using AM-GM inequality, we obtain
\begin{align}\label{eqn: term-3-bound}
    \term{4}-\sum_{t=1}^T\inner{p_t-e_{i^*}, b_t}\leq \otil\left(\frac{1}{\eta} + \sqrt{T}+ \frac{1}{\beta}+\beta T\right).
\end{align}


The final step is to handle the first term in~\pref{eqn: bias-decompose}. Similar to~\pref{eqn: omd-1}, standard analysis of online mirror descent shows that
\begin{align}\label{eqn: exp3.G}
    \sum_{t=1}^T\inner{p_t-e_{j}, \ellhat_{t}+b_t}\leq \frac{\log K}{\eta} + \eta\sum_{t=1}^T\sum_{i=1}^Kp_{t,i}\left(\ellhat_{t,i}+b_{t,i}\right)^2,
\end{align}

However, unlike the case for self-aware graphs, when there exist nodes without a self-loop, the second term may be prohibitively large when $t\in\calT$. Inspired by~\citep{alon2015online2}, we address this issue with a \emph{loss shifting trick}. Specifically, the following refined version of~\pref{eqn: exp3.G} holds:
\begin{align}\label{eqn: exp3.G-shifted}
    \sum_{t=1}^T\inner{p_t-e_{j}, \ellhat_{t}+b_t}\leq \frac{\log K}{\eta} + 2\eta\sum_{t=1}^T\sum_{i=1}^Kp_{t,i}\left(\ellhat_{t,i}+b_{t,i}-z_t\right)^2,
\end{align}
for any $z_t\leq \frac{3}{\eta}$, $t\in[T]$. We choose $z_t = 0$ when $t\notin \calT$ and $z_t=\ellhat_{t,j_t}+b_{t,j_t}$ when $t\in\calT$, which satisfies the condition $z_t\leq \frac{3}{\eta}$ again thanks to the $\order(\eta)$ amount of uniform exploration over all nodes (\pref{line: uniform}). 
With such a loss shift $z_t$, continuing with~\pref{eqn: exp3.G-shifted} it can be shown that:
\begin{align}\label{eqn: exp3.G-shifted-2}
    \sum_{t=1}^T\inner{p_t-e_{j}, \ellhat_{t}+b_t}&\leq \otil\left(\frac{1}{\eta}+\eta\sum_{t\notin\calT}\sum_{i\in [K]}p_{t,i}\ellhat_{t,i}^2+\eta\sum_{t\in\calT}\frac{\mathbbm{1}\{i_t\neq j_t\}}{W_{t,j_t}}+\beta^2T\right).
\end{align}
Note that for $t\in\calT$, $\ellhat_{t,i}\leq 2$ for all $i\in\bar{S}_t$. Therefore, the second term in~\pref{eqn: exp3.G-shifted-2} can be bounded by $\otil(\eta T+\eta\sum_{t=1}^T\sum_{i\in S_t}\frac{p_{t,i}}{W_{t,i}+\gamma}\ellhat_{t,i})\leq\otil(\eta T+\eta\sum_{t=1}^T\alpha_t+\eta/\gamma\cdot\max_{t\in[T]}\alpha_t)$ where the inequality is by using~\pref{lem: stab_freedman}. The third term can be bounded by $\otil(\sqrt{\eta T})$ with high probability by using Freedman's inequality. Together with~\pref{eqn: term-3-bound},~\pref{eqn: term-4-bound-main}, the bounds for $\term{1}$, $\term{2}$, and $\term{3}$, and the optimal choice of the parameters $\eta,\beta$, and $\gamma$, we arrive at the following main theorem for general strongly observable graphs (see \pref{app: proof-strong-alg} for the full proof).

\begin{theorem}\label{thm: main-strong-obs}
    \pref{alg: strong-obs} with parameter $\gamma = \beta = \eta = \min\left\{\frac{1}{\sqrt{\sum_{t=1}^T\alpha_t\log(1/\delta)}}, \frac{1}{2}\right\}$ guarantees that with probability at least $1-6\delta$, the regret is bounded as follows:
    \begin{align*}
        \Reg \leq \otil\left(\sqrt{\sum_{t=1}^T\alpha_t \log\frac{1}{\delta}}+\max_{t\in [T]}\alpha_t\log\frac{1}{\delta}\right).
    \end{align*}
\end{theorem}

To the best of our knowledge, this is the first optimal high-probability regret bound for general strongly observable graphs, importantly without any $\text{poly}(K)$ dependence. While in this theorem we assume the knowledge of $\alpha_t$ for all $t\in [T]$ to tune the parameters $\eta$, $\beta$ and $\gamma$, a standard doubling trick can be applied to remove this requirement and make~\pref{alg: strong-obs} completely parameter-free.\footnote{This can be achieved efficiently by applying a standard doubling trick on the quantity $B_t=\frac{1}{\sqrt{\sum_{\tau=1}^tQ_{\tau}}}$, $t\in[T]$.}

%% file: weakly.tex

\setcounter{AlgoLine}{0}
\section{High-Probability Regret for Weakly Observable Graphs}\label{sec: weakly}
\begin{algorithm}[t]
\caption{Algorithm for Weakly Observable Graphs}\label{alg: weakly}
\textbf{Input:} Parameters $\eta, \gamma, \eps$.

\textbf{Define:} Regularizer $\psi(p)=\frac{1}{\eta}\sum_{i=1}^K p_i \ln p_i$.

\textbf{Initialize:} $p_{1,i}=\frac{1}{K}$ for all $i \in [K]$.

\For{$t=1,2,\dots, T$}{

\nl Receive the feedback graph $G_t$ and (approximately) find a smallest weakly dominating set $D_t$.

\nl Let $\widetilde p_t=(1-\eps|D_t|)p_t+\eps \cdot \mathbf{1}_{D_t}$ ($\mathbf{1}_{D_t}$ is a vector with $1$ for coordinates in $D_t$ and $0$ otherwise). \label{line: uniform_dominating}

\nl Sample action $i_t \sim \widetilde p_t$.

\nl Receive feedback $\ell_{t,j}$ for all $j$ such that $i_t\in \Nin_t(j)$.

\nl
Construct estimator $\ellhat_t \in\R^K$ such that 
$\ellhat_{t,i}=\frac{\ell_{t,i}\mathbbm{1}\{i_t\in \Nin_t(i)\}}{W_{t,i}+\gamma\mathbbm{1}\{i\in S_t\}}$
where $W_{t,i}=\sum_{j\in \Nin_t (i)}\wt{p}_{t,j}$.
\label{line: loss_estimator_weakly}

\nl
Compute $p_{t+1}=\argmin_{p\in\Delta_K}\big\{\langle p,\ellhat_{t}\rangle+{\brgmd}_{\psi}(p,p_{t})\big\}$.
}
\end{algorithm}
In this section, we study the setting where the feedback graph $G_t$ is weakly observable for all $t \in [T]$. Under the uninformed setting,~\citep[Theorem 9]{alon2015online} proves that the lower bound of expected regret is $\Omega(K^{\nicefrac{1}{3}}T^{\nicefrac{2}{3}})$. To get rid of the $\poly(K)$ dependence, we thus consider the informed setting, in which $G_t$ is revealed to the learner before she selects $i_t$. We propose a simple algorithm to achieve $\otil(T^{\nicefrac{1}{3}}(\sum_{t=1}^Td_t)^{\nicefrac{1}{3}}+\frac{1}{T}\sum_{t=1}^T d_t)$ high-probability regret bound. 

Our algorithm is summarized in~\pref{alg: weakly}, which is a combination of \expthreeG~\citep{alon2015online2} and \expthreeix. Following \expthreeG, we add uniform exploration over a smallest weakly dominating set (\pref{line: uniform_dominating}).\footnote{While finding it exactly is computational hard, it suffices to find an approximate one with size $\order(d_t\log K)$, which can be done in polynomial time.} In this way, each weakly observable node has at least $\eps$ probability to be observed, which is essential to control the variance of the loss estimators. Similar to \pref{alg: strong-obs}, we apply implicit exploration for nodes with self-loops when constructing their loss estimators (\pref{line: loss_estimator_weakly}). Different from~\pref{alg: strong-obs}, we do not need to inject bias any more. This is because with the combination of uniform exploration and implicit exploration, our algorithm already achieves $\otil(T^{\nicefrac{2}{3}})$ bound which is optimal for weakly observable graphs. Our main result in summarized below (see \pref{app: weak} for the proof).
\begin{theorem}\label{thm: weak_high_prob}
Algorithm \ref{alg: weakly} with parameter $\eps=\min\{\frac{1}{2},T^{\nicefrac{1}{3}}(\sum_{t=1}^Td_t)^{-\nicefrac{2}{3}}\ln(1/\delta)^{\nicefrac{1}{3}}\}$, $\gamma=\sqrt{\frac{\ln(1/\delta)}{\sum_{t=1}^T \tilde \alpha_t}}$, $\eta=\min\left\{T^{-\nicefrac{1}{3}}(\sum_{t=1}^Td_t)^{-\nicefrac{1}{3}}\ln(1/\delta)^{-\nicefrac{1}{3}},\gamma \right\}$ ensures with probability at least $1-4\delta$:
\begin{align}
     \Reg \le \otil\pa{T^{\frac{1}{3}}\left(\sum_{t=1}^Td_t\right)^{\frac{1}{3}}\ln^{\frac{1}{3}}(1/\delta)+\frac{1}{T}\sum_{t=1}^T d_t \ln\frac{1}{\delta}+\sqrt{\sum_{t=1}^T \tilde \alpha_t \ln\frac{1}{\delta}}+\max_{t \in [T]} \tilde \alpha_t\ln\frac{1}{\delta}}, \nonumber
\end{align}
where $\tilde \alpha_t$ is the independence number of the subgraph induced by nodes with self-loops in $G_t$.
\end{theorem}
When $G_t=G$ for all $t$, our bound becomes $\otil\left(d^{\nicefrac{1}{3}}T^{\nicefrac{2}{3}}+\sqrt{\alpha T}+\alpha+d\right)$, where $\alpha$ is the independence number of the subgraph of $G$ induced by its nodes with self-loops and $d$ is the weak domination number of $G$. This even improves over the $\otil(d^{\nicefrac{1}{3}}T^{\nicefrac{2}{3}}+\sqrt{KT})$ \textit{expected} regret bound of~\citep{alon2015online2}, removing any $\text{poly}(K)$ dependence.

To prove~\pref{thm: weak_high_prob}, similar to the analysis in~\pref{sec: strongly}, we first decompose the regret against any fixed action $j\in[K]$ as follows:
\begin{align}\label{eqn: reg-decompose-weak}
    \underbrace{\sum_{t=1}^T\left(\ell_{t,i_t}-\inner{\wt{p}_t, \ell_t}\right)}_{\term{a}} +\underbrace{\sum_{t=1}^T\inner{\wt{p}_t-p_t, \ell_t}}_{\term{b}} + \underbrace{\sum_{t=1}^T\inner{p_t-e_{j}, \ell_t-\ellhat_t}}_{\term{c}} + \underbrace{\sum_{t=1}^T\inner{p_t-e_{j}, \ellhat_t}}_{\term{d}}.
\end{align}
$\term{a}$ is of order $\otil(\sqrt{T})$ via Azuma's inequality.
By the definition of $\wt{p}_t$, $\term{b}$ is of order $\order(\eps \sum_{t=1}^T d_t)$.
To bound $\term{c}$, we state the following lemma, which controls the deviation between real losses and loss estimators. The proof starts by considering nodes in $S_t$ and $\bar S_t$ separately, followed by standard concentration inequalities; see~\pref{app: weak} for details.
\begin{lemma}\label{lem: weak_deviation}
\pref{alg: weakly} guarantees the following with probability at least $1-\delta$
\begin{align}
    \sum_{t=1}^T \inner{p_t,\ell_t-\ellhat_t} \le \otil\left(\sqrt{\sum_{t=1}^T \tilde \alpha_t \ln\frac{1}{\delta}}+\gamma \sum_{t=1}^T \tilde \alpha_t+\max_{t \in [T]} \tilde \alpha_t\ln\frac{1}{\delta}+\sqrt{\frac{T}{\eps} \ln\frac{1}{\delta}}\right). \nonumber
\end{align}
Furthermore, with probability at least $1-\delta$, for any $i \in [K]$, the following inequality holds:
\begin{align}
    \sum_{t=1}^T \left(\ellhat_{t,i}-\ell_{t,i}\right) \le \otil\left(\sqrt{\frac{T}{\eps} \ln\frac{1}{\delta}}+\frac{1}{\eps}\ln\frac{1}{\delta}+\frac{1}{\gamma}\ln\frac{1}{\delta}\right). \nonumber
\end{align}
\end{lemma}
Next, we prove the following lemma bounding~$\term{d}$ (see \pref{app: weak} again for the full proof).
\begin{lemma}\label{lem: weak_stability}
\pref{alg: weakly} guarantees that with probability at least $1-\delta$
\begin{align}
    \term{d} \le \otil\left(\frac{1}{\eta}+\eta\sum_{t=1}^T \tilde \alpha_t+\frac{\eta}{\gamma}\max_t \tilde \alpha_t\ln\frac{1}{\delta}+\eta\sqrt{\frac{T}{\eps^3}\ln\frac{1}{\delta}}+\frac{\eta T}{\eps}+\frac{\eta}{\eps^2}\ln\frac{1}{\delta}\right). \nonumber
\end{align}
\end{lemma}
\begin{proofsk} First, we apply standard OMD analysis~\citep{bubeck2012regret} and obtain
\begin{align}
    \term{d} 
    &\le \frac{\log K}{\eta}+2\eta \sum_{t=1}^T \sum_{i \in S_t} p_{t,i}\hatell_{t,i}^2+
   2\eta \sum_{t=1}^T \sum_{i \in \bar S_t} p_{t,i}\hatell_{t,i}^2 \nonumber.
\end{align}
We can bound the second term by $\otil\left(\eta\sum_{t=1}^T \tilde \alpha_t+\frac{\eta}{\gamma}\max_t \tilde \alpha_t\ln(1/\delta)\right)$ using~\pref{eqn: alpha-eqn-1} in~\pref{lem: stab_freedman}. For the third term, based on the definition of $\ellhat_{t,i}$ for $i\in \bar{S}_t$, we decompose it as follows
\begin{align}
    \eta \sum_{t=1}^T \sum_{i \in \bar S_t} p_{t,i}\hatell_{t,i}^2
    \le \eta \sum_{t=1}^T \sum_{i \in \bar S_t} \frac{p_{t,i}}{ W_{t,i}}(\ellhat_{t,i}-\ell_{t,i}) + \eta \sum_{t=1}^T \sum_{i \in \bar S_t} \frac{p_{t,i}}{W_{t,i}} \ell_{t,i}. \nonumber
\end{align}
We bound the first term by $\otil(\eta\sqrt{T/\epsilon^3}+\eta/\epsilon^2)$ using Freedman's inequality. With the help of uniform exploration, we know that $W_{t,i}\geq \epsilon$ and thus the second term is bounded by $\frac{\eta T}{\eps}$.  
\end{proofsk}
With the help of~\pref{lem: weak_deviation} and~\pref{lem: weak_stability}, we are ready to prove~\pref{thm: weak_high_prob}.
\begin{proof}[Proof of~\pref{thm: weak_high_prob}]
Putting results from~\pref{eqn: reg-decompose-weak},~\pref{lem: weak_deviation}, and~\pref{lem: weak_stability} together, our regret bound becomes
\begin{align}
    \Reg &\le \otil\left(\eps \sum_{t=1}^T d_t+\sqrt{\sum_{t=1}^T \tilde \alpha_t \ln\frac{1}{\delta}}+\gamma \sum_{t=1}^T \tilde \alpha_t+\max_t \tilde \alpha_t\ln\frac{1}{\delta}+\frac{1}{\gamma}\log\frac{1}{\delta}+\sqrt{\frac{T}{\eps} \ln\frac{1}{\delta}}\right) \nonumber \\
    &+\otil\left(\frac{1}{\eta}+\eta\sum_{t=1}^T \tilde \alpha_t+\frac{\eta}{\gamma}\max_t \tilde \alpha_t\ln\frac{1}{\delta}+\eta\sqrt{\frac{T}{\eps^3}\ln\frac{1}{\delta}}+\frac{\eta T}{\eps}+\frac{\eta}{\eps^2}\ln\frac{1}{\delta}\right). \nonumber
\end{align}
By picking $\eps$, $\eta$, and $\gamma$ as stated in~\pref{thm: weak_high_prob}, we achieve that with probability at least $1-\delta$,
\begin{align}
    \Reg \le \otil\pa{T^{\frac{1}{3}}\left(\sum_{t=1}^Td_t\right)^{\frac{1}{3}}\ln^{\frac{1}{3}}\frac{1}{\delta}+\frac{1}{T}\sum_{t=1}^T d_t \ln\frac{1}{\delta}+\sqrt{\sum_{t=1}^T \tilde \alpha_t \ln\frac{1}{\delta}}+\max_t \tilde \alpha_t\ln\frac{1}{\delta}}. \nonumber
\end{align}
\end{proof}
Again, we can apply the standard doubling trick to tune $\eta$, $\gamma$, and $\eps$ adaptively without requiring the knowledge of $d_t$ and $\tilde \alpha_t$ for $t\in[T]$ ahead of time.

%% file: conclusion.tex
\section{Conclusions and Open Problems}
In this work, we design algorithms that achieve near-optimal high-probability regret bounds for adversarial MAB with time-varying feedback graphs for both the strongly observable case and the weakly observable case. We achieve $\otil((\sum_{t=1}^T\alpha_t)^{\nicefrac{1}{2}}+\max_{t\in[T]}\alpha_t)$ regret for strongly observable graphs, improving and extending the results of~\citep{neu2015explore}, which only considers self-aware graphs and suffers an $\otil(K)$ term. In addition, we derive the first high-probability regret bound for weakly observable graph setting, which also depends on $K$ only logarithmically and is order optimal.

One open problem is whether one can achieve high-probability data-dependent regret bounds for this problem, such as the so-called small-loss bounds which scales with the loss of the best action. \cite{lee2020bias} achieved expected regret bound $\otil(\sqrt{\kappa L_{\star}})$ for a fixed graph where $\kappa$ is the clique partition number and $L_{\star}$ is the loss of the best action. Achieving the same bound with high-probability under an adaptive adversary appears to require new ideas.

%% file: app_strongly.tex

\section{Omitted Details in~\pref{sec: strongly}}\label{app: strongly}

\subsection{Proof of~\pref{thm: exp3-ix}}\label{app: proof-exp3-ix}
In this section, we prove~\pref{thm: exp3-ix}, which shows that the regret of the \expthreeix algorithm~\citep{neu2015explore} does not necessarily has linear dependence on the number of actions $K$ (that appears in the original analysis), but is instead $\otil((\sum_{t=1}^T\alpha_t)^{\nicefrac{1}{2}}+\max_{t\in[T]}\alpha_t)$ with high probability. 

First, we prove~\pref{lem: stab_freedman}, which shows a tighter concentration between $\ellhat_{t}$ and $\ell_t$ and is crucial to the improvement from $\otil(K)$ to $\otil(\max_{t\in[T]}\alpha_t)$.

\begin{proof}[Proof of~\pref{lem: stab_freedman}]
    We first prove~\pref{eqn: alpha-eqn-1}. Let $X_{t,1}=\sum_{i \in S_t}\frac{p_{t,i}}{W_{t,i}+\gamma}(\ellhat_{t,i}-\ell_{t,i})$ and $Q_t=\sum_{i\in S_t}\frac{p_{t,i}}{W_{t,i}+\gamma}$. According to the definition of $\ellhat_{t,i}$ and the fact that $\ell_t\in[0,1]^K$, we know that
    \begin{align*}
        |X_{t,1}| \le \sum_{i \in S_t}\frac{p_{t,i}}{(W_{t,i}+\gamma)}\left(\frac{1}{\gamma}+1\right) \le \frac{2Q_t}{\gamma},
    \end{align*}
    where we use the fact that $\gamma\leq \frac{1}{2}$. Next, consider the term $\mathbb{E}_t[X_{t,1}^2]$.
    \begin{align}
    \E_t[X_{t,1}^2] &\le \E_t\left[\left(\sum_{i \in S_t}\frac{p_{t,i}}{W_{t,i}+\gamma}\ellhat_{t,i}\right)^2\right]\nonumber \\
    & \le \E_t\left[\left(\sum_{i \in S_t}\frac{p_{t,i}}{(W_{t,i}+\gamma)^2}\mathbbm{1}\{i_t\in \Nin_t(i)\}\right)\left(\sum_{j \in S_t}\frac{p_{t,j}}{(W_{t,j}+\gamma)^2}\mathbbm{1}\{i_t\in \Nin_t(j)\}\right)\right] \nonumber\\
    & \le \E_t\left[\left(\sum_{i \in S_t}\frac{p_{t,i}}{(W_{t,i}+\gamma)^2}\mathbbm{1}\{i_t\in \Nin_t(i)\}\right)\left(\sum_{j \in S_t}\frac{p_{t,j}}{(W_{t,j}+\gamma)^2}\right)\right] \nonumber\\
    &\leq \frac{Q_t}{\gamma}\mathbb{E}_t\left[\sum_{i\in S_t}\frac{p_{t,i}}{(W_{t,i}+\gamma)^2}\mathbbm{1}\{i_t\in \Nin_t(i)\}\right] \leq \frac{Q_t^2}{\gamma}.\nonumber
\end{align}
Note that $Q_t\leq K$. Therefore, $X_{t,1}\leq \frac{2K}{\gamma}$ and $\mathbb{E}_t[X_{t,1}^2]\leq \frac{K^2}{\gamma}$. Then, using~\pref{lem: Strengthened-Freedman}, we know that with probability at least $1-\delta$,
\begin{align*}
    \sum_{t=1}^TX_{t,1} 
    &\leq 3\sqrt{\sum_{t=1}^T\frac{Q_t^2}{\gamma}\log\left(\frac{2K}{\delta}\sqrt{\frac{T}{\gamma}}\right)} + 2\max\{1,\max_{t\in[T]}X_{t,1}\}\log\left(\frac{2K}{\delta}\sqrt{\frac{T}{\gamma}}\right). \\
    &\leq \order\left(\sum_{t=1}^T\frac{Q_t^2}{\gamma U_T} + U_T\log\left(\frac{KT}{\delta\gamma}\right)\right),
\end{align*}
where $U_T = \max\{1, \max_{t\in[T]}X_{t,1}\}$ and the last inequality is because of AM-GM inequality.

Next, we prove~\pref{eqn: alpha-eqn-2}. Let $X_{t,2}=\sum_{i \in S_t}( W_{t,i}-\mathbbm{1}\{i_t\in\Nin_t(i)\})\frac{p_{t,i}\ell_{t,i}}{ W_{t,i}+\gamma}$. Direct calculation shows that $|X_{t,2}| \le 2Q_t$. Consider its conditional variance:
\begin{align}
    \mathbb{E}_t[X_{t,2}^2] &\le \E_t\left[\left(\sum_{i \in S_t}\frac{p_{t,i}}{ W_{t,i}+\gamma}\mathbbm{1}\{i_t\in\Nin_t(i)\}\right)^2\right] \nonumber\\
    &=\sum_{i \in S_t}p_{t,i}\sum_{j \in S_t}\frac{p_{t,j}}{ W_{t,j}+\gamma} \le Q_t.\nonumber
\end{align}
Define $\iota_1=\ln\left(\frac{2\max_t Q_t+2\sqrt{\sum_{t=1}^T Q_t}}{\delta}\right)$. Applying~\pref{lem: Strengthened-Freedman}, we can obtain that with probability at least $1-\delta$,
\begin{align*}
    \sum_{t=1}^T X_{t,2} \le \order\left(\sqrt{\sum_{t=1}^T Q_t \iota_1}+\max_{t\in[T]} Q_t \iota_1\right).
\end{align*}

\end{proof}

Next, we are ready to prove~\pref{thm: exp3-ix}.

\begin{proof}[Proof of~\pref{thm: exp3-ix}]
    According to~\pref{eqn: decompose-ix}, for an arbitrary comparator $j\in [K]$, we decompose the overall regret as follows:
    \begin{align}\label{eqn: decompose-ix-app}
    \sum_{t=1}^T(\ell_{t,i_t}-\ell_{t,j}) &= \sum_{t=1}^T\inner{p_{t}-e_j,\ellhat_t} + \sum_{t=1}^T\left(\ell_{t,i_t}-\inner{p_t,\ell_t}\right) + \sum_{t=1}^T\left(\ellhat_{t,j}-\ell_{t,j}\right) \nonumber\\
    &\qquad + \sum_{t=1}^T\sum_{i=1}^K\left(W_{t,i}-\one\{i_t\in\Nin_t(i)\}\right)\frac{p_{t,i}\ell_{t,i}}{W_{t,i}+\gamma} + \sum_{t=1}^T\sum_{i=1}^K\frac{\gamma p_{t,i}\ell_{t,i}}{W_{t,i}+\gamma}.
    \end{align}

According to the standard analysis of \expthree~\citep{bubeck2012regret}, the first term of~\pref{eqn: decompose-ix-app} can be bounded as follows:
\begin{align*}
    &\sum_{t=1}^T\inner{p_{t}-e_j, \ellhat_{t}} \\
    &\leq \frac{\log K}{\eta} + \eta\sum_{t=1}^T\sum_{i=1}^Kp_{t,i}\ellhat_{t,i}^2 \\ &\leq \frac{\log K}{\eta} + \eta\sum_{t=1}^T\sum_{i=1}^K\frac{p_{t,i}}{W_{t,i}+\gamma}\ellhat_{t,i}\\
     &\leq \frac{\log K}{\eta} + \eta\sum_{t=1}^T\sum_{i=1}^K\frac{p_{t,i}}{W_{t,i}+\gamma}\ell_{t,i} + \order\left(\eta\sum_{t=1}^T\frac{Q_t^2}{\max_{\tau\in[T]}Q_\tau}+\frac{\eta\max_{t\in[T]}Q_t}{\gamma}\log\left(\frac{KT}{\delta\gamma}\right)\right),
\end{align*}
where the last inequality holds with probability at least $1-\delta$ according to~\pref{lem: stab_freedman}.

According to standard Hoeffding-Azuma inequality, we know that with probability at least $1-\delta$, the second term of~\pref{eqn: decompose-ix-app} is bounded as follows:
\begin{align*}
    \sum_{t=1}^T\left(\ell_{t,i_t}-\inner{p_t,\ell_t}\right)\leq \order\left(\sqrt{T\log\frac{1}{\delta}}\right).
\end{align*}

Based on Corollary 1 in~\citep{neu2015explore}, with probability at least $1-\delta$, the third term is bounded as follows: for all $j\in[K]$, 
\begin{align*}
    \sum_{t=1}^T\left(\ellhat_{t,j}-\ell_{t,j}\right)\leq \frac{\log(K/\delta)}{2\gamma}.
\end{align*}

The fourth term of~\pref{eqn: decompose-ix-app} can be bounded by using~\pref{lem: stab_freedman} and the final term of~\pref{eqn: decompose-ix-app} is bounded by $\order(\gamma\sum_{t=1}^TQ_t)$ (recall that $Q_t=\sum_{i\in S_t}\frac{p_{t,i}}{W_{t,i}+\gamma}$). Combining all the terms, we know that with probability at least $1-3\delta$, 
\begin{align*}
    \sum_{t=1}^T\left(\ell_{t,i_t}-\ell_{t,j}\right) &\leq \otil\left(\frac{1}{\eta} + \frac{\log(1/\delta)}{\gamma}+\left(\eta+\gamma\right)\sum_{t=1}^TQ_t+\frac{\eta\max_{t\in[T]}Q_t}{\gamma}\log\frac{1}{\delta}\right) \\
    &\qquad + \otil\left(\sqrt{T\log\frac{1}{\delta}} + \sqrt{\sum_{t=1}^T\log\frac{1}{\delta}}+\max_{t\in[T]}Q_t\log\frac{1}{\delta}\right).
\end{align*}
According to~\pref{lem: graph}, we know that $Q_t=\otil(\alpha_t)$. Finally, choosing $\eta = \gamma = \sqrt{\frac{\log(1/\delta)}{\sum_{t=1}^T\alpha_t}}$ and picking $\delta'=\frac{\delta}{3}$ finishes the proof.
\end{proof}

\subsection{Proof of~\pref{thm: main-strong-obs}}\label{app: proof-strong-alg}

In this section, we prove our main result~\pref{thm: main-strong-obs} in the strongly observable setting.
To prove~\pref{thm: main-strong-obs}, according to~\pref{eqn: strong-regret-full-decompose}, we can decompose the overall regret with respect to any $j\in[K]$ as follows

\begin{align}\label{eqn: strong-regret-full-decompose-app}
    \sum_{t=1}^T\left(\ell_{t,i_t}-\ell_{t,j}\right) &\leq \underbrace{\left(\sum_{t=1}^T\ell_{t,i_t}-\sum_{t=1}^T\inner{\wt{p}_t, \ell_t}\right)}_{\term{1}} + \underbrace{\left(\sum_{t=1}^T\inner{\wt{p}_t-p_t,\ell_t}\right)}_{\term{2}}\nonumber\\
    &\qquad + \underbrace{\left(\sum_{t=1}^T\inner{p_t-e_j, \ell_t}_{\bar{S}_t} - \sum_{t=1}^T\mathbb{E}_t\left[\inner{p_t-e_j, \ellhat_t}_{\bar{S}_t}\right] \right)}_{\term{3}} \nonumber\\
    &\qquad + \underbrace{\left(\sum_{t=1}^T\mathbb{E}_t\left[\inner{p_t-e_j, \ellhat_t}_{\bar{S}_t}\right] - \sum_{t=1}^T\inner{p_t-e_j, \ellhat_t}_{\bar{S}_t}\right)}_{\term{4}} \nonumber\\
    &\qquad + \underbrace{\left(\sum_{t=1}^T\inner{p_t-e_j,\ell_t-\ellhat_t}_{S_t}\right)}_{\term{5}} + \underbrace{\left(\sum_{t=1}^T\inner{p_t-e_j, \ellhat_t}\right)}_{\term{6}}.
\end{align}
With the help of Hoeffding-Azuma's inequality, we know that with probability at least $1-\delta$, $\term{1}\leq \order(\sqrt{T\log(1/\delta)})$. $\term{2}\leq \order(\eta T)$ because of the definition of $\wt{p}_t$ and $p_t$. $\term{3}=0$ as $\ellhat_{t,i}$ is an unbiased estimator of $\ell_{t,i}$ for $i\in\bar{S}_t$. In the next three sections, we bound $\term{4}$, $\term{5}$ and $\term{6}$ respectively.

\paragraph{Bounding~$\term{4}$.} Using Freedman's inequality, we prove the following lemma:
\begin{lemma}\label{lem: term-4}
    With probability at least $1-\delta$,
    \begin{align*}
        \term{4}\leq \left(2+\frac{4}{\eta}\right)\log\frac{1}{\delta} + 2\sqrt{\left(4T+\sum_{t=1}^T\frac{\mathbbm{1}\{t\in\calT, j\neq j_t\}}{W_{t,j_t}}\right)\log\frac{1}{\delta}}.
    \end{align*}
\end{lemma}

\begin{proof}
    Let $Y_t=\mathbb{E}_t\left[\inner{p_t-e_j,\ellhat_t}_{\bar{S}_t}\right]-\inner{p_t-e_j, \ellhat_t}_{\bar{S}_t}$. Then,
\begin{align*}
    |Y_t|\leq \mathbb{E}_t\left[\inner{p_t-e_j,\ellhat_t}_{\bar{S}_t}\right]+ 2\sum_{i\in\bar{S}_t}\frac{p_{t,i}}{\frac{K-1}{K}\eta} \leq 2 + \frac{4}{\eta}.
\end{align*}
If $t\notin \calT$, then we know that $W_{t,i}\geq 1/2$ for all $i\in \bar{S}_t$ and
\begin{align*}
    \mathbb{E}_t[Y_t^2] &\leq \mathbb{E}_t\left[\inner{p_t-e_j, \ellhat_t}_{\bar{S}_t}^2\right] \leq 4.
\end{align*}
If $t\in \calT$, then we know that $W_{t,i}\geq 1/2$ for all $i\in \bar{S}_t$ except for $i=j_t$. When $j\neq j_t$, we can bound $\mathbb{E}_t[Y_t^2]$ as follows:
\begin{align*}
    \mathbb{E}_t[Y_t^2] &\leq \mathbb{E}_t\left[\inner{p_t-e_j, \ellhat_t}_{\bar{S}_t}^2\right] \\
    &\leq 2\mathbb{E}_t\left[\inner{p_t, \ellhat_t}_{\bar{S}_t}^2\right] + 2\mathbb{E}_t\left[\inner{e_j, \ellhat_t}_{\bar{S}_t}^2\right]\\
    &\leq 2\mathbb{E}_t\left[\sum_{i\in \bar{S}_t}p_{t,i}\ellhat_{t,i}^2\right] + \frac{2}{W_{t,j_t}}\leq 2\mathbb{E}_t\left[\sum_{i\in \bar{S}_t}\frac{p_{t,i}\ellhat_{t,i}}{W_{t,i}}\right] +\frac{2}{W_{t,j_t}}\leq 4+\frac{4}{W_{t,j_t}}.
\end{align*}
If $j=j_t$, we know that $\inner{p_t-e_j, \ellhat_{t}}_{\bar{S}_t} = \sum_{i\in \bar{S}_t, i\neq j_t}p_{t,i}\ellhat_{t,i}+(1-p_{t,j_t})\cdot \frac{1}{W_{t,j_t}}\leq 2 + \frac{1}{1-\eta}\leq 4$ as $\eta\leq \frac{1}{2}$. Then we know that $\mathbb{E}_t[Y_t^2]\leq 16$.

Therefore, according to Freedman's inequality~\pref{lem: Freedman}, we know that with probability at least $1-\delta$,
\begin{align*}
    \sum_{t=1}^TY_t&\leq \min_{\lambda\in [0,1/(2+4/\eta)]}\left\{\frac{\log(1/\delta)}{\lambda}+\lambda \sum_{t=1}^T\mathbb{E}_t[Y_t^2]\right\} \\
    &\leq \min_{\lambda\in [0,1/(2+4/\eta)]}\left\{\frac{\log(1/\delta)}{\lambda}+\lambda \left(16T+4\sum_{t=1}^T\frac{\mathbbm{1}\{t\in\calT, j\neq j_t\}}{W_{t,j_t}}\right)\right\} \\
    &\leq \left(2+\frac{4}{\eta}\right)\log\frac{1}{\delta} + 2\sqrt{\left(4T+\sum_{t=1}^T\frac{\mathbbm{1}\{t\in\calT, j\neq j_t\}}{W_{t,j_t}}\right)\log\frac{1}{\delta}}.
\end{align*}
\end{proof}

\paragraph{Bounding~$\term{5}$.} The following lemma gives a bound on~$\term{5}$. The proving technique is similar to the one that we use to bound the last three terms in~\pref{eqn: decompose-ix-app}.

\begin{lemma}\label{lem: term-5}
    With probability at least $1-2\delta$,
\begin{align*}
    \term{5}\leq \otil\left(\sqrt{\sum_{t=1}^TQ_t\log\frac{1}{\delta}}+\max_{t\in[T]}Q_t\log\frac{1}{\delta}+\gamma\sum_{t=1}^TQ_t+\frac{1}{\gamma}\log\frac{1}{\delta}\right).
\end{align*}
\end{lemma}
\begin{proof}
    We bound $\sum_{t=1}^T\inner{p_t, \ell_t-\ellhat_{t}}_{S_t}$ and $\sum_{t=1}^T\inner{e_j, \ell_t-\ellhat_{t}}_{S_t}$ separately. Note that $\ellhat_{t,i}$ is an under-biased estimator of $\ell_{t,i}$ for $i\in S_t$. Direct calculation shows that
\begin{align}
    \sum_{t=1}^T\inner{p_t,\ell_t-\ellhat_t}_{S_t}&=\sum_{t=1}^T \sum_{i \in S_t}p_{t,i}(\ell_{t,i}-\hatell_{t,i})\nonumber \\
    &=\sum_{t=1}^T \sum_{i \in S_t}(W_{t,i}-\mathbbm{1}\{i_t\in\Nin_t(i)\})\frac{p_{t,i}\ell_{t,i}}{W_{t,i}+\gamma}+\sum_{t=1}^T\sum_{i \in S_t}\gamma \frac{p_{t,i}\ell_{t,i}}{W_{t,i}+\gamma}.
\end{align}
Therefore, according to~\pref{lem: stab_freedman}, with probability at least $1-\delta$,
\begin{align*}
    \sum_{t=1}^T\inner{p_t,\ell_t-\ellhat_{t}}_{S_t}&\leq\otil\left(\sqrt{\sum_{t=1}^TQ_t\log\frac{1}{\delta}}+\max_{t\in[T]}Q_t\log\frac{1}{\delta}+\gamma\sum_{t=1}^TQ_t\right).
\end{align*}

Next, consider the term $-\sum_{t=1}^T\inner{e_j, \ell_t-\ellhat_t}_{S_t}=\sum_{t=1}^T(\ellhat_{t,j}-\ell_{t,j})\cdot \mathbbm{1}\{j\in S_t\}$. Similar to the proof of Corollary 1 in~\citep{neu2015explore}, define $\bar{\ell}_{t,i}=\frac{\ell_{t,i}}{W_{t,i}}\mathbbm{1}\{i_t\in \Nin_t(i)\}$ and we know that for any $i\in [K]$,
\begin{align*}
    \ellhat_{t,i}\mathbbm{1}\{i\in S_t\} &\leq \frac{\ell_{t,i}}{W_{t,i}+\gamma}\mathbbm{1}\{i_t\in \Nin_t(i), i\in S_t\} \leq \frac{\ell_{t,i}}{W_{t,i}+\gamma\ell_{t,i}}\mathbbm{1}\{i_t\in \Nin_t(i), i\in S_t\} \\
    &\leq \frac{1}{2\gamma}\frac{2\gamma\ell_{t,i}/W_{t,i}}{1+\gamma \ell_{t,i}/W_{t,i}}\mathbbm{1}\{i_t\in \Nin_t(i), i\in S_t\} \\
    &\leq \frac{1}{2\gamma}\log\left(1+2\gamma\bar{\ell}_{t,i}\right)\mathbbm{1}\{i\in S_t\} \\
    &= \frac{1}{2\gamma}\log\left(1+2\gamma\bar{\ell}_{t,i}\mathbbm{1}\{i\in S_t\}\right).
\end{align*}
Therefore, we know that
\begin{align*}
    \mathbb{E}_t\left[\exp\left(2\gamma \ellhat_{t,i}\mathbbm{1}\{i\in S_t\}\right)\right]\leq \mathbb{E}_t\left[1+2\gamma \bar{\ell}_{t,i}\mathbbm{1}\{i\in S_t\}\right] = 1+2\gamma \ell_{t,i}\mathbbm{1}\{i\in S_t\} \leq \exp\left(2\gamma\ell_{t,i}\mathbbm{1}\{i\in S_t\}\right).
\end{align*}
Define $Z_t=\exp\left(2\gamma\mathbbm{1}\{i\in S_t\}\left(\ellhat_{t,i}-\ell_{t,i}\right)\right)$ and according to previous analysis, we know that $Z_t$ is a super-martingale and by Markov inequality, we obtain that
\begin{align*}
\Pr\left[\sum_{t=1}^T\left(\ellhat_{t,i}-\ell_{t,i}\right)\mathbbm{1}\{i\in S_t\}>\epsilon\right]=\Pr\left[\exp\left(2\gamma\sum_{t=1}^T\left(\ellhat_{t,i}-\ell_{t,i}\right)\mathbbm{1}\{i\in S_t\}\right)>\exp(2\gamma\epsilon)\right] \leq \exp(-2\gamma\epsilon).
\end{align*}
Taking a union bound over $i\in[K]$, we know that with probability at least $1-\delta$, for all $i\in [K]$,
\begin{align}\label{eqn: ix-single-bound}
    \sum_{t=1}^T\left(\ellhat_{t,i}-\ell_{t,i}\right)\mathbbm{1}\{i\in S_t\}\leq \frac{\log(K/\delta)}{2\gamma}.
\end{align}

Combining both parts gives the bound for $\term{5}$: with probability at least $1-2\delta$,
\begin{align}\label{eqn: term-5-bound}
    \term{5}\leq \otil\left(\sqrt{\sum_{t=1}^TQ_t\log\frac{1}{\delta}}+\max_{t\in[T]}Q_t\log\frac{1}{\delta}+\gamma\sum_{t=1}^TQ_t+\frac{1}{\gamma}\log\frac{1}{\delta}\right).
\end{align}
\end{proof}

\paragraph{Bounding~$\term{6}$.} For completeness, before bounding~$\term{6}$, we show the following OMD analysis lemma.
\begin{lemma}\label{lem: omd-lemma}
    Suppose that $p'=\argmin_{p\in \Delta_K}\left\{\inner{p, \ell}+D_\psi(p,p_t)\right\}$ with $\psi(p) = \frac{1}{\eta}\sum_{i=1}^Kp_i\ln p_i$. If $\eta \ell_i\geq -3$ for all $i\in [K]$, then for any $u\in\Delta_K$, the following inequality hold:
    \begin{align*}
        \inner{p-u, \ell}\leq D_{\psi}(u,p)-D_{\psi}(u,p')+2\eta\sum_{i=1}^Kp_i\ell_i^2.
    \end{align*}
\end{lemma}

\begin{proof}
    Let $q_{i} = p_i\exp(-\eta \ell_i)$ and direct calculation shows that $p'=\argmin_{p\in\Delta_K}D_{\psi}(p,q)$ and for all $u\in \Delta_K$,
    \begin{align*}
        \inner{p-u, \ell} &= D_{\psi}(u,p)-D_\psi(u,q)+D_{\psi}(p,q) \\
                          &\leq D_{\psi}(u,p)-D_\psi(u,p')+D_{\psi}(p,q),
    \end{align*}
    where the second step uses the generalized Pythagorean theorem. On the other hand, using the inequality $\exp(-x)\leq 1-x+2x^2$ for any $x\geq -3$, we know that
    \begin{align*}
        D_{\psi}(p,q) &= \frac{1}{\eta}\sum_{i=1}^K\left(p_i\ln \frac{p_i}{q_i}+q_i-p_i\right) \\
        &\leq \frac{1}{\eta}\sum_{i=1}^Kp_i\left(\exp(-\eta\ell_i)-1+\eta\ell_i\right)\leq 2\eta \sum_{i=1}^Kp_i\ell_i^2,
    \end{align*}
    where the inequality is because $\eta\ell_i\geq-3$.
\end{proof}

Now we are ready to bound~$\term{6}$ as follows.
\begin{lemma}\label{lem: term-6}
    With probability at least $1-2\delta$,
    \begin{align*}
        \term{6}&\leq \otil\left(\frac{1}{\eta} + \eta T+\eta\sum_{t=1}^TQ_t+\log\frac{1}{\delta}+48\sqrt{\eta T}+\beta^2 T+ \frac{\eta\max_{t\in[T]}Q_t}{\gamma}\log\frac{1}{\delta}\right)\\
        &\qquad -\beta\sum_{t=1}^T\frac{p_{t,j_t}}{W_{t,j_t}}\mathbbm{1}\{t\in\calT\}+\beta\sum_{t=1}^T\frac{\mathbbm{1}\{t\in\calT, j=j_t\}}{W_{t,j_t}}.
    \end{align*}
\end{lemma}

\begin{proof}
Recall that according to the definition in~\pref{alg: strong-obs}, $\Tcal\triangleq\{t\;|\;\text{there exists $j_t \in \bar S_t$ and $p_{t,j_t}>1/2$}\}$. To apply~\pref{lem: omd-lemma}, we first need to verify the scale of $\ellhat_t+b_t-z_t$ where $z_t=\ellhat_{t,j_t}+b_{t,j_t}$ if $t\in \Tcal$. If $t\notin \calT$, then we know that for all $i\in [K]$, $\eta(\ellhat_{t,i}+b_{t,i})=\eta\ellhat_{t,i}\geq 0$. If $t\in \calT$, 
    note that with an $\eta$ amount of uniform exploration,
    \begin{align*}
    \eta z_t = \eta\left(\ellhat_{t,j_t}+b_{t,j_t}\right)\leq \eta(1+\beta)\cdot\frac{1}{\frac{K-1}{K}\eta}\leq 2(1+\beta)\leq 3,
    \end{align*}    
    where the second inequality is because $K\geq 2$ and the last inequality is because $\beta\leq \frac{1}{2}$. Therefore, we know that $\eta(\ellhat_{t,i}+b_{t,i}-z_t)\geq -3$ for all $i\in[K]$.

    Therefore, applying~\pref{lem: omd-lemma} with $p=p_t$ and $p'=p_{t+1}$ and taking summation over $t\in [T]$, we know that for any $u\in\Delta_K$,
    \begin{align*}
        &\sum_{t=1}^T\inner{p_t-u, \ellhat_t+b_t} \\
        &\leq \sum_{t=1}^T\left(D_{\psi}(u,p_t)-D_{\psi}(u,p_{t+1})\right) + 2\eta\sum_{t=1}^T\sum_{i=1}^Kp_{t,i}\ellhat_{t,i}^2\cdot \mathbbm{1}\{t\notin\calT\} \\
        &\qquad+ 2\eta\sum_{t=1}^T\sum_{i=1}^Kp_{t,i}\left(\ellhat_{t,i}+b_{t,i}-z_t\right)^2\cdot \mathbbm{1}\{t\in\calT\} \\
        &\leq \frac{D_{\psi}(u,p_1)}{\eta} + 2\eta\sum_{t=1}^T\sum_{i=1}^Kp_{t,i}\ellhat_{t,i}^2\mathbbm{1}\{t\notin\calT\} + 6\eta \sum_{t=1}^T\sum_{i\ne j_t}p_{t,i}\left(\ellhat_{t,i}^2+b_{t,i}^2+z_t^2\right)\mathbbm{1}\{t\in\calT\} \\
        &= \frac{D_{\psi}(u,p_1)}{\eta} + 2\eta\sum_{t=1}^T\sum_{i=1}^Kp_{t,i}\ellhat_{t,i}^2\mathbbm{1}\{t\notin\calT\} + 6\eta \sum_{t=1}^T\sum_{i\ne j_t}p_{t,i}\left(\ellhat_{t,i}^2+z_t^2\right)\mathbbm{1}\{t\in\calT\}.
    \end{align*}
    For $t\notin \calT$, we know that $\ellhat_{t,i}\leq 2$ for $i\in \bar{S}_t$ and
    \begin{align*}
        &\sum_{i\in\bar{S}_t}p_{t,i}\ellhat_{t,i}^2\leq 4\sum_{i\notin S_t}p_{t,i}\leq 4,\\
        &\sum_{i\in S_t}p_{t,i}\ellhat_{t,i}^2\leq \sum_{i\in S_t}\frac{p_{t,i}}{W_{t,i}+\gamma}\ellhat_{t,i}.
    \end{align*}
    For $t\in \calT$, we know that $\ellhat_{t,i}\leq 2$ for all $i\in \bar{S}_t\backslash\{j_t\}$ and
    \begin{align*}
        &\sum_{i\in S_t}p_{t,i}\ellhat_{t,i}^2 \leq \sum_{i\in S_t}\frac{p_{t,i}}{W_{t,i}+\gamma}\ellhat_{t,i},\\
        &\sum_{i\in\bar{S}_t, i\ne j_t}p_{t,i}\ellhat_{t,i}^2\leq 4\sum_{i\in\bar{S}_t, i\ne j_t}p_{t,i} \leq 4,\\
        &\sum_{i\ne j_t}p_{t,i}z_t^2 \leq W_{t,j_t}\cdot\left(\ellhat_{t,j_t}+b_{t,j_t}\right)^2\leq  2W_{t,j_t}\ellhat_{t,j_t}^2 + 2\frac{\beta^2}{W_{t,j_t}}\leq 2\ellhat_{t,j_t}+\frac{4\beta^2}{\eta},
    \end{align*}
    where the last inequality uses the fact that $W_{t,j_t}\geq \frac{K-1}{K}\eta\geq \frac{1}{2}\eta$. For any $j\in[K]$, let $u=e_j\in\Delta_K$. Combining all the above inequalities, we can obtain that
    \begin{align}\label{eqn: omd-1-bias}
        &\sum_{t=1}^T\inner{p_t-e_j, \ellhat_t+b_t} \nonumber\\
        &\leq \frac{D_{\psi}(e_j,p_1)}{\eta} + 24\eta T + 6\eta\sum_{t=1}^T\sum_{i\in S_t}\frac{p_{t,i}}{W_{t,i}+\gamma}\ellhat_{t,i} + 12\eta\sum_{t=1}^T\ellhat_{t,j_t}\mathbbm{1}\{t\in \calT\} + 24\beta^2\sum_{t=1}^T\mathbbm{1}\{t\in \calT\} \nonumber\\
        &\leq \frac{D_{\psi}(e_j,p_1)}{\eta} + 24\eta T + 6\eta\sum_{t=1}^T\sum_{i\in S_t}\frac{p_{t,i}}{W_{t,i}+\gamma}\ellhat_{t,i} + 12\eta\sum_{t=1}^T\ellhat_{t,j_t}\mathbbm{1}\{t\in \calT\} + 24\beta^2T.
    \end{align}
    We first bound the term $\sum_{t=1}^T\ellhat_{t,j_t}\mathbbm{1}\{t\in\calT\}$. Let $Z_t=\ellhat_{t,j_t}\mathbbm{1}\{t\in \calT\}-\ell_{t,j_t}\mathbbm{1}\{t\in \calT\}$. We know that $\mathbb{E}_t[Z_t]=0$ and $|Z_t|\leq \frac{1}{\frac{K-1}{K}\eta}\leq \frac{2}{\eta}$. In addition,
    \begin{align*}        \mathbb{E}_t\left[Z_t^2\right] \leq \mathbb{E}_t\left[\frac{1}{W_{t,j_t}^2}\cdot \mathbbm{1}\{i_t\neq j_t\}\right]\cdot\mathbbm{1}\{t\in \calT\} = \frac{\mathbbm{1}\{t\in \calT\}}{W_{t,j_t}}.
    \end{align*}
    Therefore, by Freedman's inequality (\pref{lem: Freedman}), we can obtain that with probability at least $1-\delta$,
    \begin{align*}
        \sum_{t=1}^TZ_t &\leq \min_{\lambda\in [0,\frac{\eta}{2}]}\left\{\frac{\ln(1/\delta)}{\lambda}+\lambda\sum_{t=1}^T\mathbb{E}_t[Z_t^2]\right\} \leq \frac{2\ln(1/\delta)}{\eta} + 2\sqrt{\sum_{t=1}^T\frac{\mathbbm{1}\{t\in\calT\}}{W_{t,j_t}}}\leq \frac{2\ln(1/\delta)}{\eta} + 4\sqrt{\frac{T}{\eta}}.
    \end{align*}

    Combining with~\pref{eqn: omd-1-bias}, we know that with probability at least $1-2\delta$
    \begin{align*}
        &\sum_{t=1}^T\inner{p_t-e_j, \ellhat_t}\\
        &\leq \frac{D_{\psi}(e_j,p_1)}{\eta} + 24\eta T + 6\eta\sum_{t=1}^T\sum_{i\in S_t}\frac{p_{t,i}}{W_{t,i}+\gamma}\ellhat_{t,i}+ 24\log(1/\delta) + 48\sqrt{\eta T} + 24\beta^2T -\sum_{t=1}^T\inner{p_t-e_j,b_t}\\
        &\leq \frac{\log K}{\eta} + 24\eta T+6\eta\sum_{t=1}^TQ_t+24\log\frac{1}{\delta}+48\sqrt{\eta T}+24\beta^2 T \\
        &\qquad + \otil\left(\eta \sum_{t=1}^TQ_t + \frac{\eta\max_{t\in[T]}Q_t}{\gamma}\log\frac{1}{\delta}\right)-\sum_{t=1}^T\inner{p_t-e_j,b_t} \\
        &\leq \otil\left(\frac{1}{\eta} + \eta T+\eta\sum_{t=1}^TQ_t+\log\frac{1}{\delta}+\sqrt{\eta T}+\beta^2 T+ \frac{\eta\max_{t\in[T]}Q_t}{\gamma}\log\frac{1}{\delta}\right)\\
        &\qquad -\beta\sum_{t=1}^T\frac{p_{t,j_t}}{W_{t,j_t}}\mathbbm{1}\{t\in\calT\}+\beta\sum_{t=1}^T\frac{\mathbbm{1}\{t\in\calT, j=j_t\}}{W_{t,j_t}},
    \end{align*}
    where the second inequality is because of~\pref{lem: stab_freedman} and the choice of $p_1=\frac{1}{K}\cdot \one$.
\end{proof}

With~\pref{lem: term-4},~\pref{lem: term-5} and~\pref{lem: term-6} on hand, we are ready to prove~\pref{thm: main-strong-obs}.

\begin{proof}[Proof of~\pref{thm: main-strong-obs}]   
According to the regret decomposition~\pref{eqn: strong-regret-full-decompose-app}, \pref{lem: term-4},~\pref{lem: term-5} and~\pref{lem: term-6} and the bounds on $\term{1}$, $\term{2}$ and $\term{3}$, we know that with probability at least $1-6\delta$, for any $j\in[K]$,
\begin{align}\label{eqn: decompose-final-1}
    &\sum_{t=1}^T\left(\ell_{t,i_t}-\ell_{t,j}\right) \nonumber\\
    &\leq \otil\left(\sqrt{T\log\frac{1}{\delta}}\right) + \order(\eta T) + \left(2+\frac{4}{\eta}\right)\log\frac{1}{\delta} + 2\sqrt{\left(4T+\sum_{t=1}^T\frac{\mathbbm{1}\{t\in\calT, j\neq j_t\}}{W_{t,j_t}}\right)\log\frac{1}{\delta}} \nonumber\\
    &\qquad + \otil\left(\sqrt{\sum_{t=1}^TQ_t\log\frac{1}{\delta}}+\max_{t\in[T]}Q_t\log\frac{1}{\delta}+\gamma\sum_{t=1}^TQ_t+\frac{\log(1/\delta)}{\gamma}\right) \nonumber\\
    &\qquad + \otil\left(\frac{1}{\eta} + \eta T+\eta\sum_{t=1}^TQ_t+\log\frac{1}{\delta}+48\sqrt{\eta T}+\beta^2 T+ \frac{\eta\max_{t\in[T]}Q_t}{\gamma}\log\frac{1}{\delta}\right)\nonumber\\
    &\qquad -\beta\sum_{t=1}^T\frac{p_{t,j_t}}{W_{t,j_t}}\mathbbm{1}\{t\in\calT\}+\beta\sum_{t=1}^T\frac{\mathbbm{1}\{t\in\calT, j=j_t\}}{W_{t,j_t}} \nonumber\\
    &\leq \otil\left(\frac{1}{\eta}+\frac{\log(1/\delta)}{\gamma}+(\eta+\gamma)\sum_{t=1}^TQ_t+\sqrt{\sum_{t=1}^TQ_t\log\frac{1}{\delta}}+\sqrt{\eta T}+\beta^2T+ \left(\frac{\eta}{\gamma}+1\right)\max_{t\in[T]}Q_t\log\frac{1}{\delta}\right) \nonumber\\
    &\qquad + 2\sqrt{\left(4T+\sum_{t=1}^T\frac{\mathbbm{1}\{t\in\calT, j\neq j_t\}}{W_{t,j_t}}\right)\log\frac{1}{\delta}} -\beta\sum_{t=1}^T\frac{p_{t,j_t}}{W_{t,j_t}}\mathbbm{1}\{t\in\calT\}+\beta\sum_{t=1}^T\frac{\mathbbm{1}\{t\in\calT, j=j_t\}}{W_{t,j_t}}.
\end{align}

Consider the last three terms:
\begin{align*}
    &2\sqrt{\left(4T+\sum_{t=1}^T\frac{\mathbbm{1}\{t\in\calT, j\neq j_t\}}{W_{t,j_t}}\right)\log\frac{1}{\delta}} -\beta\sum_{t=1}^T\frac{p_{t,j_t}}{W_{t,j_t}}\mathbbm{1}\{t\in\calT\}+\beta\sum_{t=1}^T\frac{\mathbbm{1}\{t\in\calT, j=j_t\}}{W_{t,j_t}} \\
    &= 2\sqrt{\left(4T+\sum_{t=1}^T\frac{\mathbbm{1}\{t\in\calT, j\neq j_t\}}{W_{t,j_t}}\right)\log\frac{1}{\delta}} - \beta\sum_{t=1}^T\frac{p_{t,j_t}}{W_{t,j_t}}\mathbbm{1}\{t\in\calT, j\neq j_t\} \\
    &\qquad +\beta\sum_{t=1}^T\frac{(1-p_{t,j_t})\mathbbm{1}\{t\in\calT, j=j_t\}}{W_{t,j_t}}\\
    &\leq \order\left(\sqrt{T\log\frac{1}{\delta}}\right) + \frac{1}{\beta} + \beta \sum_{t=1}^T\frac{\sum_{i\neq j_t}p_{t,i}\mathbbm{1}\{t\in\calT, j=j_t\}}{(1-\eta)\sum_{i\neq j_t}p_{t,i}}\\
    &\leq \order\left(\sqrt{T\log\frac{1}{\delta}}+\frac{1}{\beta}+\beta T\right).
\end{align*}
where the first inequality uses the AM-GM inequality and the second inequality uses the fact that $\eta\leq\frac{1}{2}$. Combining with~\pref{eqn: decompose-final-1}, we obtain
\begin{align*}
    &\sum_{t=1}^T\left(\ell_{t,i_t}-\ell_{t,j}\right) \\
    &\leq \otil\left(\frac{1}{\eta}+\frac{\log(1/\delta)}{\gamma}+(\eta+\gamma)\sum_{t=1}^TQ_t+\sqrt{\sum_{t=1}^TQ_t\log\frac{1}{\delta}}+\sqrt{\eta T}+\beta^2T+ \left(\frac{\eta}{\gamma}+1\right)\max_{t\in[T]}Q_t\log\frac{1}{\delta}\right) \nonumber\\
    &\qquad + \order\left(\sqrt{T\log\frac{1}{\delta}}+\frac{1}{\beta}+\beta T\right).
\end{align*}

Using \pref{lem: graph}, we know that
\begin{align*}
    Q_t\leq 2\alpha_t\log\left(1+\frac{\lceil K^2/\gamma\rceil+K}{\alpha_t}\right)+2\leq 4\alpha_t\log\left(1+\frac{\lceil K^2/\gamma\rceil+K}{\alpha_t}\right)=\otil(\alpha_t).
\end{align*}

Picking $\eta=\beta=\gamma=1/\sqrt{\sum_{t=1}^T\alpha_t\log(1/\delta)}$, we achieve that with probability at least $1-6\delta$,
\begin{align*}
    \Reg_T\leq \otil\left(\sqrt{\sum_{t=1}^T\alpha_t\log\frac{1}{\delta}}+\max_{t\in [T]}\alpha_t\ln\frac{1}{\delta}\right).
\end{align*}
This finishes our proof.
\end{proof}

%% file: app_weakly.tex

\section{Proofs for \pref{sec: weakly}}\label{app: weak}
In this section, we prove \pref{lem: weak_deviation} and \pref{lem: weak_stability}. The key of the proof is to use a careful analysis of Freedman's inequality with the help of uniform exploration and implicit exploration.
\begin{proof}{\textbf{of \pref{lem: weak_deviation}}.} 
Recall that $\inner{p,\ell}_S\triangleq\sum_{i\in S}p_i\ell_i$ for any $S\subseteq[K]$. Therefore, we decompose the target $\sum_{t=1}^T\langle p_t, \ell_t-\ellhat_t\rangle$ as follows
\begin{align}
    \sum_{t=1}^T \inner{p_t,\ell_t-\ellhat_t}=\sum_{t=1}^T \inner{ p_t, \ell_t-\ellhat_t }_{S_t} +\sum_{t=1}^T \inner{ p_t, \ell_t-\ellhat_t }_{\bar S_t}. \nonumber
\end{align}
\paragraph{Bounding $\sum_{t=1}^T \langle p_t, \ell_t-\ellhat_t \rangle_{S_t}$:} We proceeds as follows
\begin{align}
    \sum_{t=1}^T \sum_{i \in S_t}p_{t,i}(\ell_{t,i}-\hatell_{t,i})=\sum_{t=1}^T\sum_{i \in S_t}\gamma \frac{p_{t,i}\ell_{t,i}}{W_{t,i}+\gamma}+\sum_{t=1}^T \sum_{i \in S_t}(W_{t,i}-\mathbbm{1}\{i_t\in \Nin_t(i)\})\frac{p_{t,i}\ell_{t,i}}{W_{t,i}+\gamma}. \nonumber
\end{align}
Recall that $Q_t=\sum_{i \in S_t}\frac{p_{t,i}}{ W_{t,i}+\gamma}$. As $\ell_t\in[0,1]^K$, it is clear that $\sum_{i \in S_t}\gamma \frac{p_{t,i}\ell_{t,i}}{W_{t,i}+\gamma} \le \gamma Q_t$. 
To bound the second term, according to~\pref{lem: stab_freedman}, let $\iota_1=\log\left(\frac{2\max_tQ_t+2\sqrt{\sum_{t=1}^TQ_t}}{\delta'}\right)$, we know that with probability at least $1-\delta'$,
\begin{align*}
    \sum_{t=1}^T \sum_{i\in S_t} \left(W_{t,i}-\mathbbm{1}\{i_t\in\Nin_t(i)\}\right)\frac{p_{t,i}\ell_{t,i}}{W_{t,i}+\gamma}\le \order\left(\sqrt{\sum_{t=1}^T Q_t \iota_1}+\max_{t\in[T]} Q_t \iota_1\right). 
\end{align*}
Consider the subgraph $\wt{G}_t$ of $G_t=([K], E_t)$ where $\wt{G}_t=(S_t, \wt{E}_t)$ and $\wt{E}_t\subseteq E_t$ is the set of edges with respect to the nodes in $S_t$. Applying~\pref{lem: graph} on the subgraph $\wt{G}_t$, we know that
\begin{align}\label{eqn: graph-weak}
    Q_t \leq \sum_{i\in \bar S_t}\frac{p_{t,i}}{p_{t,i}+\sum_{j\in \Nin_t(i)}p_{t,j}+\gamma}\leq \otil(\alpha_t).
\end{align}

Combining all the above equations, we know that with probability at least $1-\delta'$,
\begin{align}\label{eqn: app_weak_dev_1}
    \sum_{t=1}^T\inner{p_t, \ell_t-\ellhat_{t}}&\leq \order\left(\gamma\sum_{t=1}^TQ_t+\sqrt{\sum_{t=1}^T Q_t \iota_1}+\max_{t\in[T]} Q_t \iota_1\right) \nonumber\\
    &\leq \otil\left(\gamma\sum_{t=1}^T\wt{\alpha}_t+\sqrt{\sum_{t=1}^T \wt{\alpha}_t \iota_1}+\max_{t\in[T]} \wt{\alpha}_t \iota_1\right).
\end{align}
\paragraph{Bounding $\sum_{t=1}^T \langle p_t, \ell_t-\ellhat_t \rangle_{\bar S_t}$:} Since the loss estimators for nodes without a self-loop are unbiased, we directly apply \pref{lem: Freedman} to bound $\sum_{t=1}^T \langle p_t, \ell_t-\ellhat_t \rangle_{\bar S_t}$. Note that
\begin{align}
    \sum_{i \in \bar S_t}p_{t,i}(\ell_{t,i}-\hatell_{t,i}) &\le \sum_{i \in \bar S_t}p_{t,i}\ell_{t,i} \le 1 \nonumber \\ 
    \E_t\left[\left(\sum_{i \in \bar S_t}p_{t,i}(\ell_{t,i}-\hatell_{t,i})\right)^2\right] &\le \E_t\left[\left(\sum_{i \in \bar S_t}p_{t,i}\ellhat_{t,i}\right)^2\right] \nonumber \\
    &\leq \frac{1}{\epsilon}\mathbb{E}_t\left[\sum_{i\in\bar S_t}p_{t,i}\ellhat_{t,i}\right]\nonumber\\
    &\leq \frac{1}{\eps}\nonumber,
\end{align}
where the second inequality is because $\ellhat_{t,i}\leq \frac{1}{\epsilon}$ for all $i\in \bar S_t$. Therefore, using \pref{lem: Freedman}, we obtain that with probability at least $1-\delta'$
\begin{align}\label{eqn: app_weak_dev_2}
    \sum_{t=1}^T \sum_{i \in \bar S_t}p_{t,i}(\ell_{t,i}-\hatell_{t,i})\le 2\sqrt{\frac{T}{\eps} \ln(1/\delta')}+\ln(1/\delta'). 
\end{align}
Let $\delta'=\frac{\delta}{2}$. Combining~\pref{eqn: app_weak_dev_1},~\pref{eqn: app_weak_dev_2}, we prove the result of first part. 

For the second part, we consider the cases where $i \in S_t$ and $i \in \bar S_t$ separately.
\begin{align}
    \sum_{t=1}^T \left(\ellhat_{t,i}-\ell_{t,i}\right)=\sum_{t=1}^T \left(\ellhat_{t,i}-\ell_{t,i}\right)\mathbbm{1}\{i \in S_t\}+\sum_{t=1}^T \left(\ellhat_{t,i}-\ell_{t,i}\right)\mathbbm{1}\{i \in \bar S_t\}. \nonumber
\end{align}
The analysis for the first term is the same as~\pref{eqn: ix-single-bound} and we can obtain that with probability at least $1-\delta'$, for all $i\in [K]$,
\begin{align}\label{eqn: app_weak_dev_3}
    \sum_{t=1}^T\left(\ellhat_{t,i}-\ell_{t,i}\right)\mathbbm{1}\{i \in S_t\} \le \frac{\log(K/\delta')}{2\gamma}.
\end{align}
For the second term, note that $(\ellhat_{t,i}-\ell_{t,i})\mathbbm{1}\{i \in \bar S_t\} \le \frac{1}{\eps}$ as $\ellhat_{t,i}\leq \frac{1}{\epsilon}$ for all $i\in \bar S_t$. In addition, the conditional variance is bounded as follows
\begin{align}
    \E_t\left[\left((\ellhat_{t,i}-\ell_{t,i})\mathbbm{1}\{i \in \bar S_t\}\right)^2\right] \le \E_t\left[\ellhat_{t,i}^2\mathbbm{1}\{i \in \bar S_t\}\right] \le \frac{1}{\eps}. \nonumber
\end{align}
Using~\pref{lem: Freedman} and an union bound over $[K]$, for all $i \in [K]$, we have that with probability at least $1-\delta'$
\begin{align}\label{eqn: app_weak_dev_4}
    \sum_{t=1}^T (\ellhat_{t,i}-\ell_{t,i})\mathbbm{1}\{i \in \bar S_t\} \le \sqrt{\frac{T}{\eps}\ln\frac{K}{\delta'}}+\frac{1}{\eps}\ln\frac{K}{\delta'}.
\end{align}
Combining~\pref{eqn: app_weak_dev_3} and~\pref{eqn: app_weak_dev_4} and picking $\delta'=\frac{\delta}{2}$, we prove the second part.
\end{proof}

Next, we prove~\pref{lem: weak_stability}, which bounds the estimated regret $\term{d}$ in~\pref{eqn: reg-decompose-weak}.

\begin{proof}{\textbf{of \pref{lem: weak_stability}}.} We apply standard OMD analysis~\citep{bubeck2012regret} and obtain that
\begin{align}
    \term{d} &\le \frac{\log K}{\eta}+\eta \sum_{t=1}^T \sum_{i=1}^K p_{t,i}\ellhat_{t,i}^2 \nonumber\\
    &\le \frac{\log K}{\eta}+\eta \sum_{t=1}^T \sum_{i \in S_t} p_{t,i}\hatell_{t,i}^2+
   \eta \sum_{t=1}^T \sum_{i \in \bar S_t} p_{t,i}\hatell_{t,i}^2 \nonumber.
\end{align}
For the second term, using~\pref{lem: stab_freedman}, we know that with probability at least $1-\delta'$
\begin{align}\label{eqn: app_weak_sta_1}
    \eta \sum_{t=1}^T \sum_{i \in S_t} p_{t,i}\hatell_{t,i}^2 &\leq \eta \sum_{t=1}^T \sum_{i \in S_t} \frac{p_{t,i}}{W_{t,i}+\gamma}\hatell_{t,i} \nonumber\\
    &\leq \eta \sum_{t=1}^T \sum_{i \in S_t} \frac{p_{t,i}}{W_{t,i}+\gamma}\ell_{t,i} + \otil\left(\eta\sum_{t=1}^TQ_t+\max_{t\in[T]}\frac{Q_t}{\gamma}\log\frac{1}{\delta'}\right)\nonumber\\
    &\le \otil\left(\eta\sum_{t=1}^T \tilde \alpha_t+\frac{\eta}{\gamma}\max_t \tilde \alpha_t\ln\frac{1}{\delta'}\right).
\end{align}
For the third term, we decompose it as
\begin{align}
    \eta \sum_{t=1}^T \sum_{i \in \bar S_t} p_{t,i}\hatell_{t,i}^2&\leq \eta \sum_{t=1}^T \sum_{i \in \bar S_t} \frac{p_{t,i}}{W_{t,i}}\ellhat_{t,i} \nonumber\\
    &\le \eta \underbrace{\sum_{t=1}^T \sum_{i \in \bar S_t} \frac{p_{t,i}}{ W_{t,i}}(\ellhat_{t,i}-\ell_{t,i})}_{\term{i}} + \eta \underbrace{\sum_{t=1}^T \sum_{i \in \bar S_t} \frac{p_{t,i}}{W_{t,i}} \ell_{t,i}}_{\term{ii}}. \nonumber
\end{align}
To bound $\term{i}$, note that with uniform exploration on the dominating set, $\sum_{i \in \bar S_t}\frac{p_{t,i}}{W_{t,i}}(\ellhat_{t,i}-\ell_{t,i}) \leq \sum_{i \in \bar S_t}\frac{p_{t,i}}{W_{t,i}}\ellhat_{t,i}\leq \frac{1}{\eps^2}$. Next, we consider the conditional variance:
\begin{align}
    \E_t\left[\left(\sum_{i \in \bar S_t}\frac{p_{t,i}}{W_{t,i}}(\ellhat_{t,i}-\ell_{t,i})\right)^2\right] &\le \E_t\left[\sum_{i \in \bar S}\frac{p_{t,i}}{W_{t,i}^2}\mathbbm{1}\{i_t\in \Nin_t(i)\}\sum_{j \in \bar S_t}\frac{p_{t,j}}{W_{t,j}^2}\mathbbm{1}\{i_t\in \Nin_t(j)\}\right] \le \frac{1}{\eps^3}. \nonumber
\end{align}
Using Freedman's inequality~\pref{lem: Freedman}, we know that with probability at least $1-\delta'$,
\begin{align}
    \eta \sum_{t=1}^T \sum_{i \in \bar S_t}\frac{p_{t,i}}{W_{t,i}}(\ellhat_{t,i}-\ell_{t,i}) \leq \eta \sqrt{\frac{T}{\eps^3}\ln\frac{1}{\delta'}}+ \frac{\eta}{\eps^2}\ln\frac{1}{\delta'}. \label{eqn: app_weak_sta_2}
\end{align}
For $\term{ii}$, we directly bound it by noticing that $W_{t,i}\geq \eps$ for all $i\in\bar S_t$
\begin{align}
    \eta \sum_{t=1}^T \sum_{i \in \bar S_t} \frac{p_{t,i}}{ W_{t,i}}\ell_{t,i} \le \frac{\eta T}{\eps}. \label{eqn: app_weak_sta_3}
\end{align}
Combining~\pref{eqn: app_weak_sta_1},~\pref{eqn: app_weak_sta_2},~\pref{eqn: app_weak_sta_3} and picking $\delta'=\frac{\delta}{2}$, we finish the proof.
\end{proof}

%% file: app_aux.tex

\section{Auxiliary Lemmas}
In this section, we show several auxiliary lemmas that are useful in the analysis.

\begin{lemma}[Lemma 1 in~\citep{kocak2014efficient}]\label{lem: graph}
    Let $G=(V,E)$ be a directed graph with $|V|=K$, in which each node $i \in V$ is assigned a positive weight $w_i$. Assume that $\sum_{i \in V}w_i \le 1$, then
    \begin{align*}
       \sum_{i \in V}\frac{w_i}{w_i+\sum_{j \in \Nin(i)}w_j+\gamma} \le 2\alpha\log\left(1+\frac{\lceil K^2/\gamma \rceil+K}{\alpha}\right)+2,
    \end{align*}
    where $\alpha$ is the independence number of $G$.
\end{lemma}
\begin{lemma}[Freedman's inequality, Theorem 1~\citep{beygelzimer2011contextual}]\label{lem: Freedman}
    Let $X_1,X_2,\ldots,X_T$ be a martingale difference sequence with respect to a filtration $F_1 \subseteq F_2 \subseteq \ldots F_T$ such that $\E[X_t|F_t]=0$. Assume for all $t$, $X_t \le R$. Let $V=\sum_{t=1}^T \E[X_t^2|F_t]$. Then for any $\delta>0$, with probability at least $1-\delta$, we have the following guarantee:
    \begin{align*}
        \sum_{t=1}^T X_t &\leq \inf_{\lambda \in \left[0, \frac{\sqrt{e-2}}{R\sqrt{\ln(1/\delta)}}\right]} \left\{\sqrt{(e-2)\ln(1/\delta)}\left(\lambda V+\frac{1}{\lambda}\right)\right\} \\
        & = \inf_{ \lambda' \in [0,\frac{1}{R}]} \left\{\frac{\ln(1/\delta)}{\lambda'}+(e-2)\lambda' V\right\}.
    \end{align*}
\end{lemma}
\begin{lemma}[Strengthened Freedman's inequality, Theorem 9~\citep{zimmert2022return}]\label{lem: Strengthened-Freedman}
    Let $X_1,X_2,\ldots,X_T$ be a martingale difference sequence with respect to a filtration $F_1 \subseteq F_2 \subseteq \ldots F_T$ such that $\E[X_t|F_t]=0$ and assume $\E[|X_t||F_t] < \infty$ a.s. Then with probability at least $1-\delta$
    \begin{align*}
        \sum_{t=1}^T X_t\le 3\sqrt{V_T\log\left(\frac{2\max\{U_T,\sqrt{V_T}\}}{\delta}\right)}+2U_T\log\left(\frac{2\max\{U_T,\sqrt{V_T}\}}{\delta}\right),
    \end{align*}
    where $V_T=\sum_{t=1}^T \E[X_t^2|F_t]$, $U_T=\max\{1,\max_{s\in[T]}X_s\}$.
\end{lemma}